\pdfoutput=1
\documentclass{article}

\usepackage[final, nonatbib]{neurips_2021}

\usepackage[utf8]{inputenc} 
\usepackage[T1]{fontenc}    
\usepackage{hyperref}       
\hypersetup{colorlinks=false,linkcolor=blue}
\usepackage{url}            
\usepackage{booktabs}       
\usepackage{amsfonts}       
\usepackage{nicefrac}       
\usepackage{microtype}      
\usepackage{xcolor}         

\newcommand{\todo}[1]{\textcolor{blue}{[todo: #1]}}
\newcommand{\cameraready}[1]{#1}
\newcommand{\magenta}[1]{#1}

\usepackage{bm}
\usepackage{amsmath}
\usepackage{amsthm}
\usepackage{amssymb}
\usepackage{float}
\usepackage{makecell}
\usepackage{adjustbox}
\usepackage{multirow}
\usepackage{wrapfig}
\usepackage{subcaption}

\newtheorem{thm}{Theorem}

\newtheorem{proposition}{Proposition}
\newtheorem{property}{Property}
\newtheorem{corollary}{Corollary}
\newtheorem{lemma}{Lemma}

\newcommand{\cutsectionup}{\vspace*{-0.15in}}
\newcommand{\cutsectiondown}{\vspace*{-0.12in}}
\newcommand{\cutsubsectionup}{\vspace*{-0.1in}}

\newcommand{\cutparagraphup}{\vspace*{-0.1in}}

\title{Transformers Generalize DeepSets and Can be Extended to Graphs and Hypergraphs}

\author{
  Jinwoo Kim, Saeyoon Oh, Seunghoon Hong \\
  School of Computing, KAIST\\
  \texttt{\{jinwoo-kim, saeyoon17, seunghoon.hong\}@kaist.ac.kr} \\
}

\begin{document}

\maketitle

\begin{abstract}
    We present a generalization of Transformers to any-order permutation invariant data (sets, graphs, and hypergraphs).
    We begin by observing that Transformers generalize DeepSets, or first-order (set-input) permutation invariant MLPs.
    Then, based on recently characterized higher-order invariant MLPs, we extend the concept of self-attention to higher orders and propose higher-order Transformers for order-$k$ data ($k=2$ for graphs and $k>2$ for hypergraphs).
    Unfortunately, higher-order Transformers turn out to have prohibitive complexity $\mathcal{O}(n^{2k})$ to the number of input nodes $n$.
    To address this problem, we present sparse higher-order Transformers that have quadratic complexity to the number of input hyperedges, and further adopt the kernel attention approach to reduce the complexity to linear.
    In particular, we show that the sparse second-order Transformers with kernel attention are theoretically more expressive than message passing operations while having an asymptotically identical complexity.
    Our models achieve significant performance improvement over invariant MLPs and message-passing graph neural networks in large-scale graph regression and set-to-(hyper)graph prediction tasks.
    Our implementation is available at \url{https://github.com/jw9730/hot}.
\end{abstract}

\section{Introduction}
Graph is a universal data modality used to model social networks \cite{qiu2018deepinf}, chemical compounds \cite{gilmer2017neural}, biological structures \cite{fout2017protein}, and interactions in particle physics \cite{kipf2018neural, serviansky2020set}.
Recent graph neural networks (GNNs) adopt a message-passing scheme \cite{zhou2018graph, zhang2018deep, wu2021a}, where the node features are propagated and aggregated recurrently according to the neighborhood structure given in the input graph.
Despite the simplicity, the local and recurrent nature makes them unable to discover dependency between any two nodes with a distance longer than the message-passing steps \cite{gu2020implicit}.
The locality of message-passing is also known to be related to the over-smoothing problem that prevents scaling of GNNs \cite{li2018deeper, cai2020a, oono2020graph}.

An alternative approach is using a more general set of operations that involve global interactions while respecting the permutation symmetry of graphs.
Such operations either produce the same output regardless of the node ordering of input graph (permutation invariance), or commute with node reordering (permutation equivariance).
Message-passing is an equivariant operation, but is restricted to local neighborhoods defined by the input graph.
Recently, Maron~et.~al.~(2019)~\cite{maron2019invariant} characterized the full space of invariant and equivariant linear layers.
It turned out that these layers span not only the local neighborhood interactions of message-passing GNNs, but also global interactions such as the one between disconnected nodes, and even edge-to-node, or node-to-edge interactions (Figure~\ref{fig:mpnn_and_linear}).
\begin{figure}[!t]
    \centering
    \includegraphics[width=0.99\textwidth]{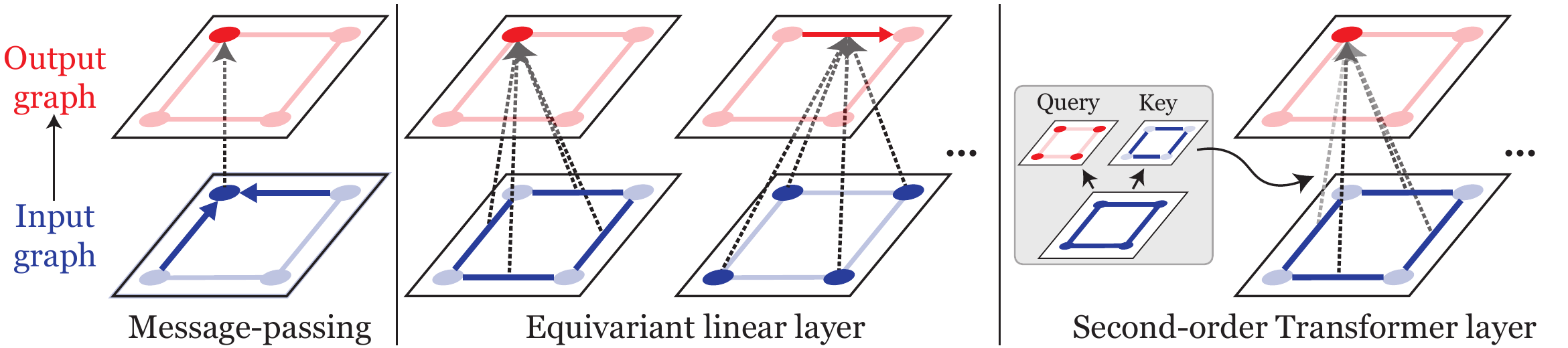}
    \caption{Illustrated operations of a message-passing GNN, an equivariant linear layer, and a second-order Transformer layer.
    A single output node is highlighted.}
    \label{fig:mpnn_and_linear}
    \vspace{-0.2cm}
\end{figure}
Notably, their formulation naturally extends to layers with different input \& output orders (e.g., edge in, node out), and higher-order layers for hypergraphs.
In theory, an invariant MLP composed of these layers should be more powerful than message-passing GNNs as they can model long-range dependency between nodes (Figure~\ref{fig:mpnn_and_linear}).
However, they are currently not widely adopted due to relatively low performance and high asymptotic memory complexity.

\paragraph{Contributions}
In this work, we present higher-order Transformers that address the low performance and high complexity of invariant MLPs.
First, based on an observation that the renowned Transformer encoder generalizes first-order equivariant linear layers or DeepSets \cite{zaheer2017deep}, we formulate higher-order Transformer layers that generalize equivariant linear layers by extending self-attention to higher orders (Figure~\ref{fig:mpnn_and_linear}).
Second, while Transformer layers with input and output orders $k,l$ has $\mathcal{O}(n^{k+l})$ asymptotic complexity, we show that by leveraging the sparsity of input hypergraphs, we can obtain $\mathcal{O}(m^2)$ complexity given input with $m$ hyperedges.
Further adopting kernel attention approaches, we propose a variant that reduces the complexity to $\mathcal{O}(m)$ and theoretically prove that they are more expressive than message passing networks.
Finally, we test higher-order Transformers on a range of tasks, and demonstrate that they achieve significant improvements over invariant MLPs and are highly competitive in performance and scalability to the state-of-the-art GNNs.
\section{Preliminary}\label{sec:preliminary}
\cutsectiondown
In this section, we first describe how (hyper)graphs can be treated as higher-order tensors.
We then describe linear layers that operate on higher-order tensors while respecting node permutation symmetry.
In particular, we analyze operations within \magenta{linear layers for second-order tensors (graphs)} and show they involve global interactions on top of local aggregation.

Let us define some notations.
We denote a set as $\{a, ..., b\}$, a tuple as $(a, ..., b)$, and denote $[n]=\{1, ..., n\}$.
We denote the space of order-$k$ tensors as $\mathbb{R}^{n^k \times d}$ where $d$ is feature dimension.
For an order-$k$ tensor $\mathbf{A}\in\mathbb{R}^{n^k\times d}$, we use multi-index $\mathbf{i}=(i_1, ..., i_k)\in[n]^k$ to denote $\mathbf{A}_{\mathbf{i}} = \mathbf{A}_{i_1, ..., i_k}\in\mathbb{R}^d$.
Let $S_n$ be the set of all permutations of $[n]$.
$\pi \in S_n$ acts on $\mathbf{i}$ by $\pi(\mathbf{i}) = (\pi(i_1), ..., \pi(i_k))$, and acts on an order-$k$ tensor $\mathbf{A}$ by $(\pi\cdot\mathbf{A})_{\mathbf{i}}=\mathbf{A}_{\pi^{-1}(\mathbf{i})}$.

\cutparagraphup
\paragraph{(Hyper)graphs as tensors}
Generally, a (hyper)graph data $G$ can be represented as a tuple $(V, \mathbf{A})$, where $V$ is a set of $n$ nodes and $\mathbf{A}\in\mathbb{R}^{n^k \times d}$ encodes features attached to hyperedges. 
The type of the hypergraph is indicated by the order $k$ of the tensor $\mathbf{A}$.
First-order tensor is a set of features (e.g., point cloud, bag-of-words) where $\mathbf{A}_i$ is the feature of node $i$.
Second-order tensor encodes edge features (e.g., adjacency) where $\mathbf{A}_{i_1, i_2}$ is the feature of edge $(i_1, i_2)$.
Generally, an order-$k$ tensor encodes hyperedge features (e.g., mesh) where $\mathbf{A}_{i_1, ..., i_k}$ is the feature of hyperedge $(i_1, ..., i_k)$.

\cutparagraphup
\paragraph{Permutation invariance and equivariance}
Our problem of interest is building a functional relation $f(\mathbf{A})\approx T$ between tensor $\mathbf{A}$ and target $T$.
If $T$ is a single output vector, we often require that $f$ is \emph{permutation invariant}, that it satisfies $f(\pi\cdot\mathbf{A}) = f(\mathbf{A})$;
if $T$ is a tensor $T=\mathbf{T}$, we often require that $f$ is \emph{permutation equivariant}, satisfying $f(\pi\cdot\mathbf{A})=\pi\cdot f(\mathbf{A})$, for all $\pi\in S_n$ and $\mathbf{A}\in\mathbb{R}^{n^k\times d}$.
In a typical design setup where a neural network $f$ is built using linear layers and non-linear activations, the construction of $f$ reduces to finding invariant and equivariant \emph{linear} layers.

\cutparagraphup
\paragraph{Invariant and equivariant linear layers}\label{sec:equivariant_linear}
\begin{figure}[!t]
    \centering
    \includegraphics[width=0.99\textwidth]{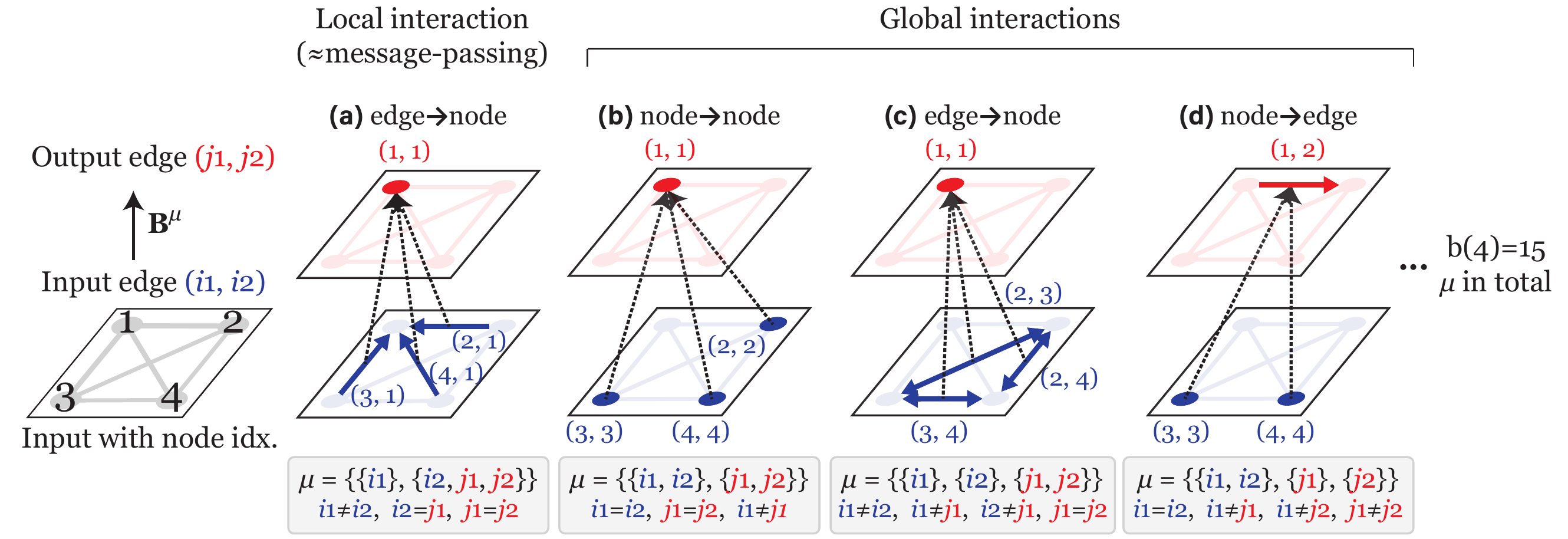}
    \caption{
    Example operations in a second-order layer $L_{2\rightarrow 2}$.
    We illustrate how input edges $(i_1,i_2)=\mathbf{i}$ are aggregated to an output edge $(j_1,j_2)=\mathbf{j}$ in various equivalence classes $(\mathbf{i},\mathbf{j})\in\mu$ via basis tensor $\mathbf{B}^\mu_{\mathbf{i},\mathbf{j}}$.
    Note that a loop ($(i_1, i_1)$ or $(j_1, j_1)$) represents a node.
    }
    \label{fig:equivariant_linear_layers}
    \vspace{-0.3cm}
\end{figure}
We summarize invariant linear layers $L_{k\rightarrow 0}: \mathbb{R}^{n^k \times d}\rightarrow\mathbb{R}^{d'}$ and equivariant linear layers $L_{k\rightarrow l}: \mathbb{R}^{n^k\times d}\rightarrow\mathbb{R}^{n^l\times d'}$ identified in Maron~et.~al.~(2019)~\cite{maron2019invariant}.
Note that invariant layer is a special case of $L_{k\rightarrow l}$ with $l=0$.
In summary, given an input $\mathbf{A}\in\mathbb{R}^{n^k\times d}$, the order-$l$ output of an equivariant layer $L_{k\rightarrow l}$ can be written with $\mathbf{i}\in[n]^k,\mathbf{j}\in[n]^l$ as:
\begin{align}\label{eqn:equivariant_layer}
    L_{k\rightarrow l}(\mathbf{A})_{\mathbf{j}} = \sum_{\mu}{\sum_{\mathbf{i}}{\mathbf{B}^{\mu}_{\mathbf{i}, \mathbf{j}}\mathbf{A}_{\mathbf{i}}w_{\mu}}} + \sum_{\lambda}{\mathbf{C}^{\lambda}_{\mathbf{j}}b_{\lambda}},
\end{align}
where $w_\mu\in\mathbb{R}^{d\times d'}$, $b_\lambda\in\mathbb{R}^{d'}$ are weight and bias parameters, $\mathbf{B}^\mu\in\mathbb{R}^{n^{k+l}}$ and $\mathbf{C}^\lambda\in\mathbb{R}^{n^l}$ are \emph{basis tensors} (will be defined), and $\mu$ and $\lambda$ are \emph{equivalence classes} of order-$(k+l)$ and order-$l$ multi-indices, respectively.
The equivalence classes are defined upon equivalence relation $\sim$ that, for multi-indices $\mathbf{i}, \mathbf{j}\in[n]^k$, $\mathbf{i}\sim\mathbf{j}$ iff $(i_1, ..., i_k) = (\pi(j_1), ..., \pi(j_k))$ for some node permutation $\pi\in S_n$.
\magenta{Then, a multi-index $\mathbf{i}$ and all members $\mathbf{j}$ in its equivalence class have the same (permutation-invariant) equality pattern: $\mathbf{i}_{a} = \mathbf{i}_{b}\Leftrightarrow\mathbf{j}_{a} = \mathbf{j}_{b}$ for all $a, b\in [k]$.}
Consequently, each equivalence class $\mu$ (or $\lambda$) is a distinct set of all order-$(k+l)$ (or order-$l$) multi-indices having a specific equality pattern.

Notably, we can represent each equivalence class of order-$k$ multi-indices as a unique partition of $[k]$ regardless of $n$, where the partition specifies the equality pattern.
\emph{e.g.,} with $k=2$, we have two partitions and respective equivalent classes: $\mu_1=\{\{i_1, i_2\}\}$ the set of all $(i_1, i_2)$ with $i_1=i_2$, and $\mu_2=\{\{i_1\}, \{i_2\}\}$ the set of all $(i_1, i_2)$ with $i_1\neq i_2$.
Thus, with $\text{b}(k)$ the $k$-th Bell number or number of partitions of $[k]$, we have $\text{b}(k+l)$ equivalence classes $\mu$ for the weight and $\text{b}(l)$ equivalence classes $\lambda$ for the bias.
We guide the readers interested in derivation to Maron~et.~al.~(2019)~\cite{maron2019invariant}.

With the equivalence classes, basis tensors are then defined as follows:
\begin{align}\label{eqn:basis_tensor_weight_bias}
    \begin{array}{ll}
        \mathbf{B}_{\mathbf{i}, \mathbf{j}}^{\mu} = \left\{
        \begin{array}{cc}
            1   &  \scalebox{0.95}{$(\mathbf{i}, \mathbf{j})\in\mu$} \\
            0   &  \scalebox{0.95}{\text{ otherwise}}
        \end{array}\right.;&
        \mathbf{C}_{\mathbf{j}}^{\lambda} = \left\{
        \begin{array}{cc}
            1   &  \scalebox{0.95}{$\mathbf{j}\in\lambda$}\\
            0   &  \scalebox{0.95}{\text{ otherwise}}
        \end{array}\right.
    \end{array}
\end{align}

In Eq.~\eqref{eqn:equivariant_layer}, each equivalence class $\mu$ determines which (hyper)edges participate and how they interact in summation $\sum_\mathbf{i}\mathbf{B}_{\mathbf{i},\mathbf{j}}^\mu\mathbf{A}_\mathbf{i}$.
As an example, let us consider $L_{2\rightarrow 2}$ that maps input edges $\mathbf{i}=(i_1,i_2)$ to output edges $\mathbf{j}=(j_1,j_2)$.
An equivalence class $\mu = \{\{i_1, i_2\}, \{j_1, j_2\}\}$ represents all $(\mathbf{i},\mathbf{j})\in\mu$ that $i_1=i_2$, $j_1=j_2$, and $i_1\neq j_1$.
Then, due to masking by $\mathbf{B}^\mu$, only input elements $\mathbf{A}_{i_1,i_2}$ with $i_1=i_2$ participate in the summation and gives output $L_{2\rightarrow 2}(\mathbf{A})_{j_1, j_2}$ for $j_1=j_2$, $i_1 \neq j_1$. 
Intuitively, this is analogous to computing a node feature by sum-pooling all the other nodes (Fig.~\ref{fig:equivariant_linear_layers}(b)).
Different $\mu$ accounts for other interactions as shown in Figure~\ref{fig:equivariant_linear_layers}.
Notably, it contains a richer set of operations beyond local interactions modeled by message-passing (Fig.~\ref{fig:equivariant_linear_layers}(a)), such as global interaction across all nodes (Fig.~\ref{fig:equivariant_linear_layers}(b)), and edge-to-node (Fig.~\ref{fig:equivariant_linear_layers}(c)), node-to-edge interactions (Fig.~\ref{fig:equivariant_linear_layers}(d)).

We finish the section by writing out the first-order equivariant layer $L_{1\rightarrow 1}$.
As $\text{b}(2)=2$, the layer has two equivalence classes $\mu_1=\{\{i_1, j_1\}\}$ and $\mu_2=\{\{i_1\}, \{j_1\}\}$.
Then, we have $\mathbf{B}^{\mu_1}=I_n$ and $\mathbf{B}^{\mu_2}=1_n1_n^\top-I_n$, with $1_n\in\mathbb{R}^n$ vector of ones.
Then, given a set of features $\mathbf{A}\in\mathbb{R}^{n\times d}$,
\begin{align}
    L_{1\rightarrow 1}(\mathbf{A}) &= I_n\mathbf{A}w_1'+(1_n1_n^\top-I_n)\mathbf{A}w'_2+1_nb^{\top}\\
    &= I_n\mathbf{A}w_1+1_n1_n^\top\mathbf{A}w_2+1_nb^{\top},
    \label{eqn:deepset}
\end{align}
where $w_1,w_2,w'_1,w'_2\in\mathbb{R}^{d\times d'}$, $b\in\mathbb{R}^{d'}$.
$L_{1\rightarrow 1}$ is analogous to a combination of elementwise feedforward ($\mu_1$) and sum-pooling of set elements ($\mu_2$), and is also known as a DeepSet layer \cite{zaheer2017deep}.

\cutsectionup
\section{Higher-Order Transformers}\label{sec:higher_order_transformer}
\cutparagraphup
In Section~\ref{sec:preliminary}, we introduced higher-order linear equivariant layers $L_{k\to l}$, and showed that they contain various global and node/edge interactions that are not covered by message-passing.
In this section, we establish a connection between the first-order equivariant layer $L_{1\to 1}$ and self-attention of Transformer encoder layers \cite{vaswani2017attention}.
Then, we extend the relationship to higher orders by tensorizing queries and keys, and formulate higher-order Transformer layers.
We finish the section by proposing a principled parameter reduction for queries and keys, which reduces a fair amount of computation.

\subsection{Transformers generalize DeepSets}\label{sec:transformers_generalize_deepsets}
\cutparagraphup
As shown in Section~\ref{sec:preliminary}, first-order linear layer, or DeepSet, has a simple structure composed of feedforward and sum-pooling (Eq.~\eqref{eqn:deepset}).
Although it is theoretically proven to be a universal approximator of permutation-invariant functions \cite{zaheer2017deep}, static sum-pooling could be limited in capturing interactions of set elements, motivating the use of sophisticated pooling.
In particular, the self-attention mechanism of Transformer encoder \cite{vaswani2017attention} was shown to achieve a large performance gain in set modeling via context-aware weighted pooling \cite{lee2019set, yun2020are}.
To see this, let us write out the Transformer encoder layers.

A Transformer encoder layer is a function $\text{Enc}:\mathbb{R}^{n\times d }\to\mathbb{R}^{n\times d}$ consisting of two layers: a self-attention layer $\text{Attn}:\mathbb{R}^{n\times d }\to\mathbb{R}^{n\times d}$ and an elementwise feedforward layer $\text{MLP}:\mathbb{R}^{n\times d }\to\mathbb{R}^{n\times d}$.
For a set of $n$ input vectors $X\in\mathbb{R}^{n\times d}$, a Transformer layer computes the following:
\begin{align}
    \text{Attn}(X)_i &= X_i + \sum_{h=1}^H{\sum_{j=1}^n{\alpha^h_{i j}X_j w_h^V w_h^O}},\\
    \text{Enc}(X)_i &= \text{Attn}(X)_i +
    \text{MLP}(\text{Attn}(X))_i,
\end{align}
where $\alpha^h=\sigma(X w_h^Q(X w_h^K)^\top)$ is an attention coefficient with an activation $\sigma$, $H$ is the number of heads, $d_H$ is head size, $d_F$ is hidden dimension, and $w_h^O\in\mathbb{R}^{d_H\times d}$, $w_h^V,w_h^K,w_h^Q\in\mathbb{R}^{d\times d_H}$.\footnote{Note that we omitted normalization after $\text{Attn}(\cdot)$ and $\text{MLP}(\cdot)$ for simplicity as in \cite{yun2020are, hanin2017approximating}.}

Now, we show that Transformer layers are generalized first-order linear equivariant functions. 
By setting $\alpha_{ij}^h=1$ and assuming that $\text{MLP}(Y)$ approximates a linear layer $Yw^F+b^{F\top}$ following the universal approximation theorem \cite{hornik1989multilayer}, the Transformer layer reduces to the following
\footnote{In practice, Transformer employs softmax in attention and deviates from Deepsets due to normalization.}
:
\begin{align}
    \text{Enc}(X)_i &= X_i (I_n + w^F) + \sum_{j=1}^n{X_j} w^{VO}(I_n + w^F) + b^{F\top},
\end{align}
where $w^{VO}=\sum_{h=1}^H{w_h^V w_h^O}$.
This is equivalent to a DeepSet layer in Eq.~\eqref{eqn:deepset} with $w_1=I_n+w^F$, $w_2=w^{VO}(I_n+w^F)$, and $b = b^F$.
In other words, we can convert DeepSets to Transformers by changing the static pooling to attention and replacing elementwise linear mapping with an MLP.
We generalize this approach to higher-order input and output, which leads to the formulation of higher-order Transformers for graphs and hypergraphs.

\subsection{Higher-order Transformer layers}\label{sec:subsec_higher_order_transformer_layers}
\cutparagraphup
In Section~\ref{sec:transformers_generalize_deepsets}, we showed that Transformer layers are generalized first-order equivariant linear layers $L_{1\to 1}$.
Notably, the generalization procedure was equivalent to changing static pooling to attention and adding feedforward MLP.
In this section, we generalize the approach to $L_{k\to l}$ with arbitrary orders $k$ and $l$ and formulate higher-order Transformer layer $\text{Enc}_{k\to l}$.

In general, we define higher-order Transformer layer \magenta{as} a function $\text{Enc}_{k\to l}:\mathbb{R}^{n^k\times d}\to\mathbb{R}^{n^l\times d}$ \magenta{with} two layers: a higher-order self-attention layer $\text{Attn}_{k\to l}:\mathbb{R}^{n^k\times d }\to\mathbb{R}^{n^l\times d}$ and a feedforward layer $\text{MLP}_{l\to l}:\mathbb{R}^{n^l\times d }\to\mathbb{R}^{n^l\times d}$.
For an input tensor $\mathbf{A}\in\mathbb{R}^{n^k\times d}$, a Transformer layer computes:
\begin{align}
    \text{MLP}_{l\to l}(\text{Attn}_{k\to l}(\mathbf{A})) &= L_{l\to l}^2\left(\text{ReLU}(L_{l\to l}^1(\text{Attn}_{k\to l}(\mathbf{A})))\right),\label{eqn:higher_mlp}\\
    \text{Enc}_{k\to l}(\mathbf{A}) &= \text{Attn}_{k\to l}(\mathbf{A}) + \text{MLP}_{l\to l}(\text{Attn}_{k\to l}(\mathbf{A})),\label{eqn:higher_transformer_layer}
\end{align}
where $L_{l\to l}^1:\mathbb{R}^{n^{l}\times d\to n^{l}\times d_F}$ and $L_{l\to l}^2:\mathbb{R}^{n^{l}\times d_F\to n^{l}\times d}$ are equivariant linear layers with hidden dimension $d_F$.
Remaining question is how to define and compute higher-order self-attention $\text{Attn}_{k\to l}(\mathbf{A})$.

To design $\text{Attn}_{k\to l}$, we remove the bias from Eq.~\eqref{eqn:equivariant_layer} and introduce attention coefficients.
It is achieved by changing each $\mathbf{B}^\mu\in\mathbb{R}^{n^{k+l}}$ to an attention coefficient tensor $\boldsymbol{\alpha}^{h,\mu}\in\mathbb{R}^{n^{k+l}}$ with multiple heads:
\begin{align}\label{eqn:equivariant_layer_attcoef}
    \text{Attn}_{k\to l}(\mathbf{A})_{\mathbf{j}} &= \sum_{h=1}^H{\sum_{\mu}{\sum_{\mathbf{i}}{\boldsymbol{\alpha}^{h,\mu}_{\mathbf{i}, \mathbf{j}}\mathbf{A}_{\mathbf{i}}w^V_{h,\mu}w^O_{h,\mu}}}},
\end{align}
where $w^O_{h,\mu}\in\mathbb{R}^{d_H\times d}$, $w^V_{h,\mu}\in\mathbb{R}^{d\times d_H}$ are learnable parameters, 
$H$ denotes the number of heads, and $d_H$ denotes head size.
Then similar to first-order case (Section~\ref{sec:transformers_generalize_deepsets}), we can show the following:

\begin{thm}\label{thm:generalization}
$\text{Enc}_{k\to l}$ (Eq.~\eqref{eqn:higher_transformer_layer}) \magenta{is a generalization of} $L_{k\to l}$ (Eq.~\eqref{eqn:equivariant_layer}).
\end{thm}
\cutsectionup
\begin{proof}
Let $\boldsymbol{\alpha}^{h,\mu}_{\mathbf{i},\mathbf{j}}=1$ for all $h,\mu$, and \scalebox{0.95}{$(\mathbf{i},\mathbf{j})\in\mu$}.
This renders $\boldsymbol{\alpha}^{h,\mu} = \mathbf{B}^\mu$ from the definition of $\mathbf{B}^\mu$.
Additionally, let $\text{MLP}_{l\to l}(\text{Attn}_{k\to l}(\mathbf{A}))_\mathbf{j} = \sum_{\lambda}{\mathbf{C}^{\lambda}_{\mathbf{j}}b_{\lambda}}$ for all $\mathbf{j}\in[n]^l$.
That is, $\text{MLP}_{l\to l}$ ignores input and reduces to an invariant bias in Eq.~\eqref{eqn:equivariant_layer}.
Then, Eq.~\eqref{eqn:higher_transformer_layer} reduces to the following:
\begin{align}
    \text{Enc}_{k\to l}(\mathbf{A})_\mathbf{j} &= \sum_{\mu}{\sum_{\mathbf{i}}{\mathbf{B}_{\mathbf{i}, \mathbf{j}}^\mu\mathbf{A}_{\mathbf{i}}\sum_{h=1}^H{w^V_{h,\mu}w^O_{h,\mu}}}} + \sum_{\lambda}{\mathbf{C}^{\lambda}_{\mathbf{j}}b_{\lambda}},
\end{align}
which is equivalent to Eq.~\eqref{eqn:equivariant_layer} with $w_\mu = \sum_{h=1}^H{w^V_{h,\mu}w^O_{h,\mu}}$.
\end{proof}

Now, we describe how to compute each attention tensor $\boldsymbol{\alpha}^{\mu}\in\mathbb{R}^{n^{k+l}}$ from input $\mathbf{A}\in\mathbb{R}^{n^k\times d}$ (Eq.~\eqref{eqn:equivariant_layer_attcoef}, we drop head index $h$ for brevity).
We obtain each attention tensor from higher-order query and key:
\begin{align}\label{eqn:attention_qk_redundant}
    \boldsymbol{\alpha}_{\mathbf{i},\mathbf{j}}^{\mu} = \left\{\begin{array}{cc}
    \sigma(\mathbf{Q}_{\mathbf{j}}^{\mu}, \mathbf{K}_{\mathbf{i}}^{\mu})/Z_{\mathbf{j}} & \text{\scalebox{0.95}{$(\mathbf{i},\mathbf{j})\in\mu$}}\\
    0 & \text{\scalebox{0.95}{ otherwise}}
    \end{array} \right. \text{where }
    \mathbf{Q}^{\mu}&=L^{\mu}_{k\to l}(\mathbf{A}), \mathbf{K}^{\mu}=L^{\mu}_{k\to k}(\mathbf{A}).
\end{align}
where $Z_{\mathbf{j}}=\sum_{\mathbf{i}|(\mathbf{i}, \mathbf{j})\in\mu} \sigma(\mathbf{Q}_{\mathbf{j}}^{\mu}, \mathbf{K}_{\mathbf{i}}^{\mu})$ is a normalization constant.
Note that query and key tensors are computed from the input $\mathbf{A}$ using the equivariant linear layers in Eq.~\eqref{eqn:equivariant_layer}.
This leads to permutation equivariance (or invariance) of Transformer encoder layer $\text{Enc}_{k\to l}$ in Eq.~\eqref{eqn:higher_transformer_layer}.

\magenta{\paragraph{Reducing orders of query and key}}
Although Eq.~\eqref{eqn:equivariant_layer_attcoef} and Eq.~\eqref{eqn:attention_qk_redundant} provide a simple and generic definition of higher-order self-attention, we observe that there exist a lot of unnecessary computations.
Specifically, there exist elements of query $\mathbf{Q}^\mu$ and key $\mathbf{K}^\mu$ that are unused in computation of attention coefficient $\boldsymbol{\alpha}^\mu_{\mathbf{i},\mathbf{j}}$, as it depends only on indices satisfying $(\mathbf{i},\mathbf{j})\in\mu$.
In fact, it turns out that the effective orders of query and key are much smaller than $l$ and $k$, as we show below:

\begin{proposition}\label{proposition:compactness}
From Eq.~\eqref{eqn:attention_qk_redundant}, let $u(\cdot)$ denote the number of unique entries in a multi-index.
With $u_q=u(\mathbf{j})$, $u_k=u(\mathbf{i})$ for some $(\mathbf{i}, \mathbf{j})\in\mu$, we can always find suitable linear layers $L^{\mu}_{k\to u_q}$, $L^{\mu}_{k\to u_k}$ and index space mappings $f_q^\mu: [n]^{l}\to[n]^{u_q}$, $f_k^\mu: [n]^{k}\to[n]^{u_k}$ that satisfy the following.
\begin{align}\label{eqn:attention_qk_compact}
    \boldsymbol{\alpha}_{\mathbf{i},\mathbf{j}}^{\mu} &= \sigma(\tilde{\mathbf{Q}}_{\mathbf{j}'}^{\mu}, \tilde{\mathbf{K}}_{\mathbf{i}'}^{\mu})/\tilde{Z}_{\mathbf{j}}\ \forall(\mathbf{i}, \mathbf{j})\in\mu,
\end{align}
where $\tilde{Z}_{\mathbf{j}}=\sum_{\mathbf{i}|(\mathbf{i}, \mathbf{j})\in\mu} \sigma(\tilde{\mathbf{Q}}_{\mathbf{j}'}^{\mu}, \tilde{\mathbf{K}}_{\mathbf{i}'}^{\mu})$, $\tilde{\mathbf{Q}}^{\mu} = L^{\mu}_{k\to u_q}(\mathbf{A})$, $\tilde{\mathbf{K}}^{\mu} = L^{\mu}_{k\to u_k}(\mathbf{A})$, $\mathbf{j}' = f_q^\mu(\mathbf{j})$, $\mathbf{i}' = f_k^\mu(\mathbf{i})$.
\end{proposition}
\cutsectionup
\begin{proof}
We leave the proof in Appendix~\ref{sec:proof_proposition_compactness}.
\end{proof}

Based on Proposition~\ref{proposition:compactness}, we can compute query and key in Eq.~\eqref{eqn:attention_qk_redundant} in a much compact way using linear layers with output orders $u_q$ and $u_k$ instead of $l$ and $k$.
In our experiments, we observe that this optimization is very useful in reducing the number of parameters and memory footprint to a feasible level without affecting the effective model capacity.

\section{Asymptotically Efficient Higher-Order Transformers}
\label{sec:implementation}
\cutsectiondown
In Section~\ref{sec:higher_order_transformer}, we formulated higher-order Transformer layers $\text{Enc}_{k\to l}$ and showed that they generalize linear equivariant layers $L_{k \to l}$.
However, this capability comes with a cost; the high asymptotic complexity of the Transformer encoder limits its practical merits.
Specifically, we show the following:
\begin{property}\label{property:complexity}
Given input size $n$,
the asymptotic complexity of a linear layer $L_{k\to l}$ (Eq.~\eqref{eqn:equivariant_layer}) is $\mathcal{O}(n^{k+l})$, and complexity of an encoder layer $\text{Enc}_{k\to l}$ (Eq.~\eqref{eqn:higher_transformer_layer}) is $\mathcal{O}(n^{k+l}+n^{2k}+n^{2l})$. 
\end{property}
\cutsectionup
\begin{proof}
We leave the proof in Appendix~\ref{sec:proof_property_complexity}.
\end{proof}
Thus, in this section, we \magenta{further} analyze the encoder layer and propose a number of optimization and relaxation to reduce the asymptotic complexity \magenta{with a minimal impact on capability}.
Notably, combining all \magenta{our} strategies reduces the complexity to $\mathcal{O}(m)$ given a hypergraph with $m$ hyperedges.
Even with such efficiency, we show that the reduced version of higher-order Transformer is \magenta{theoretically more expressive} than all message-passing neural networks.

\subsection{Linear layers with reduced complexity}
\label{sec:lightweight}
\cutparagraphup
A major computation bottleneck within Transformer encoder layer $\text{Enc}_{k\to l}$ is the higher-order linear layer, since it is used for key, query, and feedforward layer.
By exploiting only a subset of basis, we show that we can implement lightweight version of linear layer with reduced asymptotic complexity.
\begin{proposition}\label{proposition:lightweight}
Given a linear layer $L_{k\to l}$ in Eq.~\eqref{eqn:equivariant_layer}, we can always find a nonempty subset $\mathcal{M}$ of equivalence classes such that computation of the following for all $\mathbf{j}$ has $\mathcal{O}(n^l)$ complexity.
\begin{align}\label{eqn:lightweight}
    \bar{L}_{k\to l}(\mathbf{A})_{\mathbf{j}} = \sum_{\mu\in \mathcal{M}}{\sum_{\mathbf{i}}{\mathbf{B}^{\mu}_{\mathbf{i}, \mathbf{j}}\mathbf{A}_{\mathbf{i}}w_{\mu}}} + \sum_{\lambda}{\mathbf{C}^{\lambda}_{\mathbf{j}}b_{\lambda}},
\end{align}
\end{proposition}
\cutsectionup
\begin{proof}
We leave the proof in Appendix~\ref{sec:proof_proposition_lightweight}.
\end{proof}
\cutsubsectionup
\vspace{-0.1cm}
We term the reduced linear layer $\bar{L}_{k\to l}$ in Eq.~\eqref{eqn:lightweight} as a \emph{lightweight} linear layer.
In practice, we choose $\mathcal{M}$ among equivalence classes that do not involve summation over input.
For instance, for $\bar{L}_{1\to 1}$, we use only the basis for elementwise mapping ($I_n$) and drop sum-pooling ($1_n1_n^\top$) from Eq.~\eqref{eqn:deepset}.
This approximation effectively reduces the complexity of linear layers at the cost of losing inter-element dependencies. 
We employ the lightweight linear layers within the $\text{Enc}_{k\to l}$ (Eq.~\eqref{eqn:higher_transformer_layer}) to compute query and key embeddings, while the element dependency within $\text{Enc}_{k\to l}$ is handled by the higher-order self-attention using all equivalence classes as in Eq.~\eqref{eqn:attention_qk_redundant}.
This design is coherent to original (first-order) Transformers, where the elements are first linearly projected to query/key with elementwise basis ($I_n$) and interaction is handled by attention (implicitly $1_n1_n^\top$).
Importantly, this does not hurt Theorem~\ref{thm:generalization} as attention coefficients can still reduce to one and $\text{MLP}$ can reduce to bias.

\vspace{-0.05cm}
By using lightweight linear layers in $\text{Enc}_{k\to l}$, we can significantly reduce the computational cost.
The complexity of $L^\mu_{k\to u_k}$, $L^\mu_{k\to u_q}$, and $\text{MLP}_{l\to l}$ in Eq.~\eqref{eqn:attention_qk_compact} and Eq.~\eqref{eqn:higher_transformer_layer} reduces to $\mathcal{O}(n^k)$, $\mathcal{O}(n^l)$, and $\mathcal{O}(n^l)$, respectively, and as a result $\text{Enc}_{k\to l}$ becomes $\mathcal{O}(n^{k+l})$.
As a result, we have higher-order Transformer layers $\text{Enc}_{k\to l}$ that generalize $L_{k\to l}$ while retaining the complexity $\mathcal{O}(n^{k+l})$.
From here we assume that all linear layers for key, query, and MLP are lightweight.

\vspace{-0.3cm}
\subsection{Sparse Transformer layers}\label{sec:sparse}
\cutparagraphup
Even with lightweight linear layers, $\mathcal{O}(n^{k+l})$ complexity of $\text{Enc}_{k\to l}$ is still impractical.
For example, for graphs, complexity larger than $\mathcal{O}(n^2)$ is regarded prohibitive while $\text{Enc}_{2\to 2}$ is $\mathcal{O}(n^4)$.
Fortunately, leveraging the sparsity inherent in real-world graphs can significantly reduce the complexity.
As we show, it reduces the complexity of $\text{Enc}_{k\to l}$ to $\mathcal{O}(m^2)$ for processing a hypergraph with $m$ \magenta{hyperedges}.

Let $E$ the set of hyperedges of input hypergraph $G$.
Each hyperedge is generally represented by a multi-index $\mathbf{i}\in[n]^k$, so we denote $E=\{\mathbf{i}_1,...,\mathbf{i}_m\}$ with $m$ hyperedges.
Leveraging sparsity of $E$ is straightforward when the order of the hypergraph is fixed (e.g., $L_{k\to k}$); we can perform computations with only respect to the existing hyperedges $\mathbf{i},\mathbf{j}\in E$.
However, in our framework, order can change by layers (\emph{e.g.}, $L_{k\to l}$), making it difficult to directly transfer the sparsity structure $E$ to different-order output.
Inspired by network projection~\cite{carletti2020random}, we remedy this by constructing $E'$ such that for any $\mathbf{j}\in E'$, there exists $\mathbf{i}\in E$ containing all unique elements of $\mathbf{j}$.
For example, for $L_{3\to 2}$, this corresponds to projection of third-order $E$ to second-order $E'$ by obtaining edges from sides of triangles.
Despite simplicity, this simple heuristic works well in general and generalizes to any order. 

Then, we integrate the hyperedge sets $E$, $E'$ into computation of linear layer in Eq.~\eqref{eqn:lightweight} as follows:
\begin{align}\label{eqn:sparse_equivariant_layer}
    \bar{L}_{k\to l}(\mathbf{A}, E)_{\mathbf{j}} &= \left\{\begin{array}{cc}\sum_{\mu\in \mathcal{M}}{\sum_{\mathbf{i}\in E}{\mathbf{B}^{\mu}_{\mathbf{i}, \mathbf{j}}\mathbf{A}_{\mathbf{i}}w_{\mu}}} + \sum_{\lambda}{\mathbf{C}^{\lambda}_{\mathbf{j}}b_{\lambda}}
    &
    \text{\scalebox{0.95}{$\mathbf{j}\in E'$}}\\
    0&\text{\scalebox{0.95}{ otherwise}}\end{array} \right.
\end{align}

Likewise, integrating $E$, $E'$  into attention computation in Eq.~\eqref{eqn:equivariant_layer_attcoef} and Eq.~\eqref{eqn:attention_qk_redundant} we have:
\begin{align}
    \text{Attn}_{k\to l}(\mathbf{A}, E)_{\mathbf{j}} &= \left\{\begin{array}{cc}\sum_{h=1}^H{\sum_{\mu}{\sum_{\mathbf{i}\in E}{\boldsymbol{\alpha}^{h,\mu}_{\mathbf{i}, \mathbf{j}}\mathbf{A}_{\mathbf{i}}w^V_{h,\mu}w^O_{h,\mu}}}}
    &
    \text{\scalebox{0.95}{$\mathbf{j}\in E'$}}\\
    0&\text{\scalebox{0.95}{ otherwise}}\end{array} \right.\label{eqn:sparse_self_attention}\\
    \text{where } \boldsymbol{\alpha}_{\mathbf{i},\mathbf{j}}^{h,\mu} &= \left\{\begin{array}{cc} \sigma(\mathbf{Q}_\mathbf{j}^{h,\mu}, \mathbf{K}_\mathbf{i}^{h,\mu})/Z_\mathbf{j} &
        \text{\scalebox{0.95}{$(\mathbf{i},\mathbf{j})\in\mu, \mathbf{i}\in E, \mathbf{j}\in E'$}}
        \\
        0 & \text{\scalebox{0.95}{ otherwise}}
    \end{array} \right.,\label{eqn:sparse_attention_coefficient}
\end{align}
where $\mathbf{Q}^{h,\mu}=\bar{L}^{h,\mu}_{k\to l}(\mathbf{A}, E)$, $\mathbf{K}^{h,\mu}=\bar{L}^{h,\mu}_{k\to k}(\mathbf{A}, E)$, $Z_\mathbf{j} = \sum_{\mathbf{i}|(\mathbf{i},\mathbf{j})\in\mu\wedge\mathbf{i}\in E}\sigma(\mathbf{Q}_{\mathbf{j}}^{h, \mu}, \mathbf{K}_{\mathbf{i}}^{h, \mu})$.

With the computations, we can show the following:
\begin{property}\label{property:sparse_complexity}
When given $E$ with $m$ elements, the equivariant linear layer in Eq.~\eqref{eqn:sparse_equivariant_layer} has $\mathcal{O}(m)$ complexity, and self-attention computation in Eq.~\eqref{eqn:sparse_self_attention} has $\mathcal{O}(m^2)$ complexity.
Consequently, the computation done by a $\text{Enc}_{k\to l}$ composed of the layers has $\mathcal{O}(m^2)$ complexity.
\end{property}
\cutsectionup
\begin{proof}
We leave the proof in Appendix~\ref{sec:proof_property_sparse_complexity}.
\end{proof}

\subsection{Kernel attention trick}
\label{sec:kernel_linear}
\cutparagraphup
Section~\ref{sec:sparse} shows that, by constraining linear layers within $\text{Enc}_{k\to l}$ to have sparse input and output, we can reduce $\mathcal{O}(n^{k+l})$ complexity to quadratic $\mathcal{O}(m^2)$ to input size.
Yet, even $\mathcal{O}(m^2)$ can be demanding with large or dense input.
As the quadratic term comes from self-attention computation, we follow the prior work in kernel attention \cite{katharopoulos2020transformers, choromanski2020rethinking} and view attention coefficients as pairwise dot-product scores.
As we will show, this allows us to further reduce the complexity of $\text{Enc}_{k\to l}$ to $\mathcal{O}(n^k+n^l)$ and even $\mathcal{O}(m)$ for sparse version, at the cost of relaxing some modeling assumption.

We begin by approximating attention coefficient in Eq.~\eqref{eqn:attention_qk_redundant} using pairwise dot-product kernel \cite{katharopoulos2020transformers, choromanski2020rethinking}:
\begin{align}\label{eqn:attention_qk_kernel}
    \boldsymbol{\alpha}_{\mathbf{i},\mathbf{j}}^{\mu} = \left\{\begin{array}{cc} \phi(\mathbf{Q}^{\mu}_\mathbf{j})^\top \phi(\mathbf{K}^{\mu}_\mathbf{i})/Z_\mathbf{j} &
        \text{\scalebox{0.95}{$(\mathbf{i},\mathbf{j})\in\mu$}}\\
        0 &
        \text{\scalebox{0.95}{ otherwise}}
    \end{array}\right. \text{ where }Z_{\mathbf{j}}=\sum_{\mathbf{i}|(\mathbf{i},\mathbf{j})\in\mu} \phi(\mathbf{Q}^{\mu}_\mathbf{j})^\top \phi(\mathbf{K}^{\mu}_\mathbf{i})
\end{align}
where $\phi:\mathbb{R}^{d_H}\to\mathbb{R}_+^{d_K}$ is kernel feature map.
The choice of kernel can be flexible, and in our implementation we adopt Performer kernel \cite{choromanski2020rethinking} that has strong theoretical and empirical guarantee.

Substituting Eq.~\eqref{eqn:attention_qk_kernel} in Eq.~\eqref{eqn:equivariant_layer_attcoef}, we have:
\begin{align}
    \text{Attn}_{k\to l}(\mathbf{A})_{\mathbf{j}} &= \sum_{\mu}{Z_\mathbf{j}^{-1}}\sum_{\mathbf{i}|(\mathbf{i},\mathbf{j})\in\mu}{\phi(\mathbf{Q}}^{\mu}_{\mathbf{j}})^\top \phi(\mathbf{K}^{\mu}_{\mathbf{i}})
    \mathbf{A}_{\mathbf{i}}w^V_{\mu}w^O_{\mu}\nonumber\\
    &= \sum_{\mu}{Z_\mathbf{j}^{-1}}{\phi(\mathbf{Q}}^{\mu}_{\mathbf{j}})^\top \sum_{\mathbf{i}|(\mathbf{i},\mathbf{j})\in\mu} \phi(\mathbf{K}^{\mu}_{\mathbf{i}})\mathbf{A}_{\mathbf{i}}w^V_{\mu}w^O_{\mu},
    \label{eqn:dotproductattn}\\
    &\text{ where }Z_{\mathbf{j}} 
    =
    \phi(\mathbf{Q}^{\mu}_\mathbf{j})^\top
    \sum_{\mathbf{i}|(\mathbf{i},\mathbf{j})\in\mu}\phi(\mathbf{K}^{\mu}_\mathbf{i}).
    \label{eqn:dotproductnorm}
\end{align}
Here, the inner-summations in Eq.~\eqref{eqn:dotproductattn} and Eq.~\eqref{eqn:dotproductnorm} are over the key index $\mathbf{i}$, which is \emph{coupled} with the query index $\mathbf{j}$ by $(\mathbf{i},\mathbf{j})\in\mu$.
This coupling causes the major computational bottleneck, since the inner-summations should be computed for every $\mathbf{j}$-th query.
We propose an approximation by decoupling the key and query indices and taking inner-summations over ${\mathcal{I}}=\bigcup_\mathbf{j}{\{\mathbf{i}|(\mathbf{i},\mathbf{j})\in\mu\}}$:
\begin{align}\label{eqn:kernel_querywise}
    \text{Attn}_{k\to l}(\mathbf{A})_{\mathbf{j}}\approx \sum_{\mu}{Z_\mathbf{j}^{-1}}\phi({\mathbf{Q}}^{\mu}_{\mathbf{j}})^\top \sum_{{\mathbf{i}\in\mathcal{I}}}{\phi({\mathbf{K}}^{\mu}_{\mathbf{i}})
    \mathbf{A}_{\mathbf{i}}w^V_{\mu}w^O_{\mu}},
    \text{ where }Z_{\mathbf{j}}\approx
    \phi(\mathbf{Q}^{\mu}_\mathbf{j})^\top
    \sum_{\mathbf{i}\in\mathcal{I}}\phi(\mathbf{K}^{\mu}_\mathbf{i}).
\end{align}

The approximation allows a query to attend to some additional keys (depending on $\mu$), but it does not hurt the central requirement of attention that each query can assign different attention weights to the keys.
With the approximation, we can compute the summations $\sum_{{\mathbf{i}\in\mathcal{I}}}{\phi({\mathbf{K}}^{\mu}_{\mathbf{i}})}\mathbf{A}_{\mathbf{i}}$ and $\sum_{{\mathbf{i}\in\mathcal{I}}}{\phi({\mathbf{K}}^{\mu}_{\mathbf{i}})}$ only once and reuse them across all query indices $\mathbf{j}$ to reduce the cost.
Specifically, we show:

\begin{property}\label{property:kernel_complexity}
The encoder Enc$_{k\to l}$ with approximation in Eq.~\eqref{eqn:kernel_querywise} has a complexity of $\mathcal{O}(n^k+n^l)$.
Exploiting sparsity further reduces the complexity to $\mathcal{O}(m)$, linear to the number of hyperedges $m$.
\end{property}
\cutsectionup
\begin{proof}
We leave the proof in Appendix~\ref{sec:proof_property_kernel_complexity}.
\end{proof}

\cutsectionup
\subsection{Theoretical analysis and comparison to message-passing}\label{sec:message_passing}
\cutparagraphup
We showed that exploiting sparsity and adopting kernel attention reduces the computational complexity of the Transformer encoder to linear to input edges.
Considering graphs ($k=l=2$), this complexity is equivalent to or better than $\mathcal{O}(n+m)$ complexity of the message passing operation\footnote{In practice, we place node features on the diagonals of the adjacency matrix, leading to $\mathcal{O}(n+m)$.}.
Still, we show that our (approximate) encoder is theoretically more expressive than message passing.

Specifically, we show the following:
\begin{thm}\label{thm:message_passing}
A composition of two sparse Transformer layers $\text{Enc}_{2\to 2}$ with kernel attention can approximate any message passing algorithms (Gilmer~et.~al.~(2017)~\cite{gilmer2017neural}) to arbitrary precision, while the opposite is not true.
\end{thm}
\cutsectionup
\begin{proof}
We leave the proof in Appendix~\ref{sec:proof_theorem_message_passing}.
\end{proof}
\vspace{-0.2cm}
This leads to the following corollary:
\begin{corollary}
Second-order sparse Transformers with kernel attention are more expressive than any message-passing neural networks within the framework of Gilmer~et.~al.~(2017)~\cite{gilmer2017neural}.
\end{corollary}
\vspace{-0.1cm}
In the proof, we note that local information propagation in message-passing GNNs is carried out in $\text{Enc}_{2\to 2}$ by a single $\mu=\{\{i_1\}, \{i_2, i_3, i_4\}\}$.
Other types of $\mu$ would carry out different operations, which provides intuition on the powerfulness of Transformers.

\section{Experiments}\label{sec:experiments}
\cutparagraphup
In this section, we demonstrate the capability of higher-order Transformers on a variety of tasks including synthetic data, large-scale graph regression, and set-to-(hyper)graph prediction.
Specifically, we use a synthetic node classification dataset from Gu~et.~al.~(2020)~\cite{gu2020implicit}, a molecular graph regression dataset from Hu~et.~al.~(2021)~\cite{hu2021ogb}, two set-to-graph prediction datasets from Serviansky~et.~al.~(2020)~\cite{serviansky2020set}, and three hyperedge prediction datasets used in Zhang~et.~al.~(2020)~\cite{zhang2020hyper}.
Details including the datasets and hyperparameters can be found in Appendix~\ref{sec:experimental_details}.
We implemented invariant MLP in Maron~et.~al.~(2019)~\cite{maron2019invariant} as one of baselines, which we abbreviate as $\text{MLP}_{\pi}$.
We build our model by gradually adding lightweight linear layers (Sec.~\ref{sec:lightweight}), sparse linear layers (Sec.~\ref{sec:sparse}), and kernel attention (Sec.~\ref{sec:kernel_linear}) and denote them by (D), (S), and ($\phi$), respectively.

\cutparagraphup
\paragraph{Runtime and memory analysis}
\begin{figure}[!t]
    \vspace{-0.2cm}
    \centering
    \includegraphics[width=0.99\textwidth]{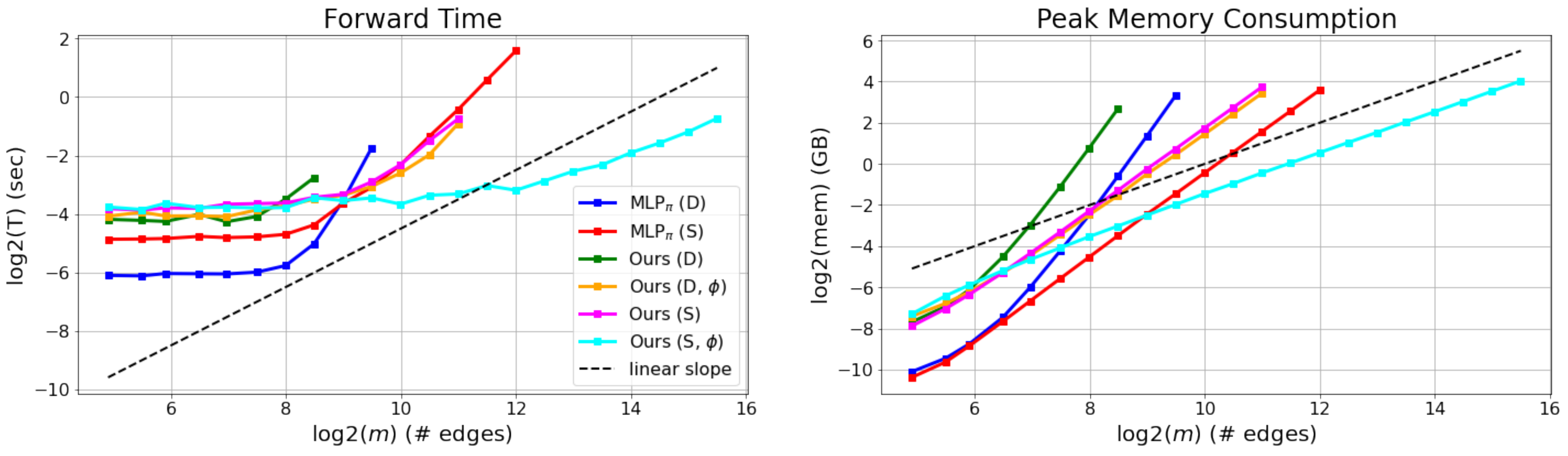}
    \vspace{-0.2cm}
    \caption{
    Comparison of all second-order models in terms of forward time, memory consumption, and maximal possible input size.
    Plots are shown until each model runs into out-of-memory error on a RTX 6000 GPU with 22GB.
    }
    \label{fig:cost_plot}
    \vspace{-0.4cm}
\end{figure}
To experimentally verify our claims on linear complexity in Sec.~\ref{sec:kernel_linear}, we conducted a runtime and memory consumption analysis on all second-order models using random graphs;
details can be found in Appendix~\ref{sec:runtime_memory_details}.
The results are shown in Figure~\ref{fig:cost_plot}.
Consistent with our theoretical claims, sparse second-order Transformer with kernel attention (Ours~(S,~$\phi$)) is the only variant that linearly scales to input size in terms of both time and memory.
Also, it is the only one that successfully scales to graphs with $>50\text{k}$ edges, while still very fast in smaller graphs.

\cutparagraphup
\paragraph{Synthetic chains}
To test the ability of higher-order Transformers in modeling long-range interactions in graphs, we used a synthetic dataset where the task is node classification in chain graphs.
Binary class information is provided in a terminal node, so a model is required to propagate the information across the chain by handling long-range dependency.
We used training and test sets with chains of length 20 and 200 respectively.
As baselines, we used 3 message-passing networks, a sparse invariant MLP, and ablated versions of sparse second-order Transformers where $\mu$ for global pooling are removed (w/o global).
Further details can be found in Appendix~\ref{sec:chains_details}.

\begin{wraptable}{r}{0.46\textwidth}
\vspace{-0.2cm}
\caption{Chain node classification results.}
\vspace{-0.1cm}
\centering
\begin{adjustbox}{width=0.46\textwidth}
    \label{table:chain}
        \begin{tabular}{lcc}
        \Xhline{2\arrayrulewidth}
        \\[-1em] Method & Micro-$F_1$ (\%) & Macro-$F_1$ (\%) \\
        \\[-1em]\Xhline{2\arrayrulewidth}
        \\[-1em] GCN & $47.78\pm4.17$ & $33.58\pm1.86$ \\
        \\[-1em] GIN-0 & $53.72\pm4.17$ & $36.22\pm1.86$ \\
        \\[-1em] GAT & $47.78\pm4.17$ & $33.58\pm1.86$ \\
        \\[-1em]\Xhline{2\arrayrulewidth}
        \\[-1em] $\text{MLP}_\pi$ (S) & $53.5\pm4.16$ & $36.04\pm1.97$ \\
        \\[-1em] Ours (S) w/o global & $53.72\pm4.17$ & $36.22\pm1.86$ \\
        \\[-1em] Ours (S, $\phi$) w/o global & $50.77\pm5.15$ & $35.22\pm2.17$ \\
        \\[-1em]\Xhline{2\arrayrulewidth}
        \\[-1em] Ours (S) & $\mathbf{100\pm0}$ & $\mathbf{100\pm0}$ \\
        \\[-1em] Ours (S, $\phi$) & $\mathbf{100\pm0}$ & $\mathbf{100\pm0}$ \\
        \\[-1em]\Xhline{2\arrayrulewidth}
    \end{tabular}
\end{adjustbox}
\end{wraptable}

The test performances are in Table~\ref{table:chain}.
We first note that second-order Transformers successfully capture long-range dependency up to 200 nodes apart, while message-passing networks fail.
Importantly, when the subset of basis that accounts for global pooling is eliminated, the performance of Transformers drops similar to message-passing networks, showing their importance in modeling long-range interaction.
Yet, simply having global basis is not enough, as seen in the failure of $\text{MLP}_\pi$.
This indicates fine-grained interaction modeling via attention is essential even in this simple task.

\cutparagraphup
\paragraph{Large-scale graph regression}
To further evaluate higher-order Transformers in large-scale setting, we used the PCQM4M-LSC dataset from Open Graph Benchmark \cite{hu2021ogb}, which is the largest graph-level regression dataset composed of $3.8\text{M}$ molecular graphs. 
As test data is unavailable, we report the Mean Absolute Error (MAE) on validation dataset.
In addition to the baselines from the benchmark, we also report the performances of second-order invariant MLP and a vanilla (first-order) Transformer\footnote{As vanilla Transformer operates on node features only, we used Laplacian graph embeddings \cite{belkin2003laplacian, dwivedi2020a} as positional embeddings so that the model can consider edge structure information.} for comparison.
Further details can be found in Appendix~\ref{sec:regression_details}.

\begin{wraptable}{r}{0.35\textwidth}
\caption{PCQM4M-LSC large-scale graph regression results. * indicates results are obtained with a shorter schedule ($10\%$ of the full iterations).}
\centering
\begin{adjustbox}{width=0.35\textwidth}
    \label{table:graph_regression}
        \begin{tabular}{lc}
        \Xhline{2\arrayrulewidth}\\[-1em]
        \\[-1em] Model & Validate MAE \\
        \\[-1em]\Xhline{2\arrayrulewidth}\\[-1em]
        \\[-1em] MLP-FINGERPRINT (\cite{hu2021ogb}) & 0.2044 \\
        \\[-1em] GCN (\cite{hu2021ogb}) & 0.1684 \\
        \\[-1em] GIN (\cite{hu2021ogb}) & 0.1536 \\
        \\[-1em] GCN-VN (\cite{hu2021ogb}) & 0.1510 \\
        \\[-1em] GIN-VN (\cite{hu2021ogb}) & 0.1396 \\
        \\[-1em]\Xhline{2\arrayrulewidth}\\[-1em]
        \\[-1em] Transformer + Laplacian PE* & 0.2162 \\
        \\[-1em] $\text{MLP}_\pi$ (S)* & 0.1464 \\
        \\[-1em] Ours (S, $\phi$)$_{-\text{SMALL}}$* & 0.1376 \\
        \\[-1em] Ours (S, $\phi$)* & 0.1294 \\
        \\[-1em]\Xhline{2\arrayrulewidth}\\[-1em]
        \\[-1em] Ours (S, $\phi$) & \textbf{0.1263} \\
        \\[-1em]\Xhline{2\arrayrulewidth}
    \end{tabular}
\end{adjustbox}
\end{wraptable}

The results are in Table~\ref{table:graph_regression}.
Second-order Transformer outperforms the message-passing GNNs (GCN, GIN) by a large margin, including the ones with a virtual node that can model long-range interactions (GCN-VN, GIN-VN).
It suggests that higher-order attention is potentially better in handling long-range interactions on graphs than the current practice of augmenting GNNs with a virtual node.
Furthermore, second-order Transformer outperforms invariant MLP, indicating that replacing sum-pooling with attention is important for scale-up.
Finally, second-order Transformer significantly outperforms vanilla Transformer with Laplacian graph embeddings.
This is presumably because node embeddings are insufficient to utilize features associated with edges, while second-order Transformers can naturally use all edge information.
We also note that invariant MLP and vanilla Transformer have worse complexity than the second-order Transformer ($\mathcal{O}(m^2)$ and $\mathcal{O}(n^2)$, respectively), while the second-order Transformer has $\mathcal{O}(m)$ complexity identical to GCN.

\cutparagraphup
\vspace{-0.05cm}
\paragraph{Set-to-graph prediction}
An important advantage of our framework distinguished from most existing GNNs is that, by design, it can be applied to any learning scenario with different input and output orders (mixed-order).
To demonstrate this, we tested higher-order Transformers in set-to-graph prediction tasks where the goal is to predict edge structure of a graph given a set of node features.
We used two datasets following the prior work \cite{serviansky2020set}.
The first dataset Jets originates from particle physics experiments, where collision of high-energy particles gives a set of observed particles.
The task is to partition the feature set of observed particles according to their origin.
By viewing each subset of particles as a fully-connected graph, the problem is cast as a set-to-graph prediction.
For the second dataset Delaunay, the task is to predict Delaunay triangulation \cite{deluanay1934sur} given a set of points in 2D space.
Two datasets are used for this task, one containing 50 points and the other containing varying number of points $\in\{20, ..., 80\}$.
For the baselines, we take the scores reported in Serviansky~et.~al.~(2020)~\cite{serviansky2020set}, which includes GNNs and invariant MLPs (S2G, S2G+) with $L_{1\rightarrow 1}$ and $L_{1\rightarrow 2}$.
Our model is made by substituting the linear layers with $\text{Enc}_{1\rightarrow 1}$ and $\text{Enc}_{1\rightarrow 2}$ (mixed-order) respectively.
Further details including the datasets, metrics, and baselines can be found in Appendix~\ref{sec:set2graph_details}.

\begin{table}[b!]
\vspace{-0.4cm}
\caption{Set-to-graph results. Lower-right panel shows Delaunay (20-80) sample from ours and S2G.}
\vspace{0.1cm}
\label{table:set2graph}
\centering
\begin{adjustbox}{width=0.79\textwidth}
    \begin{tabular}{cl|ccc|cl|cccc}
        \Xhline{2\arrayrulewidth}\\[-1em]
        \\[-1em]& Method & F1 & RI & ARI & & Method & Acc & Prec & Rec & F1 \\
        \\[-1em]\Xhline{2\arrayrulewidth}
        \\[-1em]\multirow{9}{3em}{\makecell{Jets\\(B)}}
        & AVR & 0.565 & 0.612 & 0.318 & \multirow{9}{4em}{\makecell{Delaunay\\(50)}} & SIAM & 0.939 & 0.766 & 0.653 & 0.704\\
        & MLP & 0.533 & 0.643 & 0.315 & & SIAM-3 & 0.911 & 0.608 & 0.538 & 0.570 \\
        & SIAM & 0.606 & 0.675 & 0.411 & & GNN0 & 0.826 & 0.384 & 0.966 & 0.549\\
        & SIAM-3 & 0.597 & 0.673 & 0.396 & & GNN5 & 0.809 & 0.363 & \textbf{0.985} & 0.530 \\
        & GNN & 0.586 & 0.661 & 0.381 & & GNN10 & 0.759 & 0.311 & 0.978 & 0.471\\
        & S2G & 0.646 & 0.736 & 0.491 & & S2G & 0.984 & 0.927 & 0.926 & 0.926 \\
        & S2G+ & 0.655 & 0.747 & 0.508 & & S2G+ & 0.983 & 0.927 & 0.925 & 0.926 \\
        \cline{2-5}\cline{7-11}\\[-1em]
        & Ours (D) & 0.667 & 0.746 & 0.520 & & Ours (D) & \textbf{0.994} & \textbf{0.981} & 0.967 & \textbf{0.974} \\
        & Ours (D, $\phi$) & \textbf{0.670} & \textbf{0.751} & \textbf{0.526} & & Ours (D, $\phi$) & 0.991 & 0.967 & 0.952 & 0.959 \\
        \Xhline{2\arrayrulewidth}
        \\[-1em]\multirow{9}{3em}{\makecell{Jets\\(C)}}
        & AVR & 0.695 & 0.650 & 0.326 & \multirow{9}{4em}{\makecell{Delaunay\\(20-80)}} & SIAM & 0.919 & 0.667 & 0.764 & 0.687 \\
        & MLP & 0.686 & 0.658 & 0.319 & & SIAM-3 & 0.895 & 0.578 & 0.622 & 0.587 \\
        & SIAM & 0.729 & 0.695 & 0.406 & & GNN0 & 0.810 & 0.387 & 0.946 & 0.536\\
        & SIAM-3 & 0.719 & 0.710 & 0.421 & & GNN5 & 0.777 & 0.352 & \textbf{0.975} & 0.506 \\
        & GNN & 0.720 & 0.689 & 0.390 & & GNN10 & 0.746 & 0.322 & 0.970 & 0.474 \\
        & S2G & 0.747 & 0.727 & 0.457 & & S2G & 0.947 & 0.736 & 0.934 & 0.799 \\
        & S2G+ & 0.751 & 0.733 & 0.467 & & S2G+ & 0.947 & 0.735 & 0.934 & 0.798 \\
        \cline{2-5}\cline{7-11}\\[-1em]
        & Ours (D) & 0.755 & 0.732 & 0.469 & & Ours (D) & \textbf{0.993} & \textbf{0.982} & 0.960 & \textbf{0.971} \\
        & Ours (D, $\phi$) & \textbf{0.757} & \textbf{0.735} & \textbf{0.473} & & Ours (D, $\phi$) & 0.989 & 0.948 & 0.956 & 0.952 \\
        \Xhline{2\arrayrulewidth}
        \\[-1em]\multirow{9}{3em}{\makecell{Jets\\(L)}}
        & AVR & 0.970 & 0.965 & 0.922  & \hspace{0cm}
        \multirow{9}{4em}{
        \begin{minipage}{0.55\textwidth}
            \vspace{-0.1cm}
            \includegraphics[width=\linewidth]{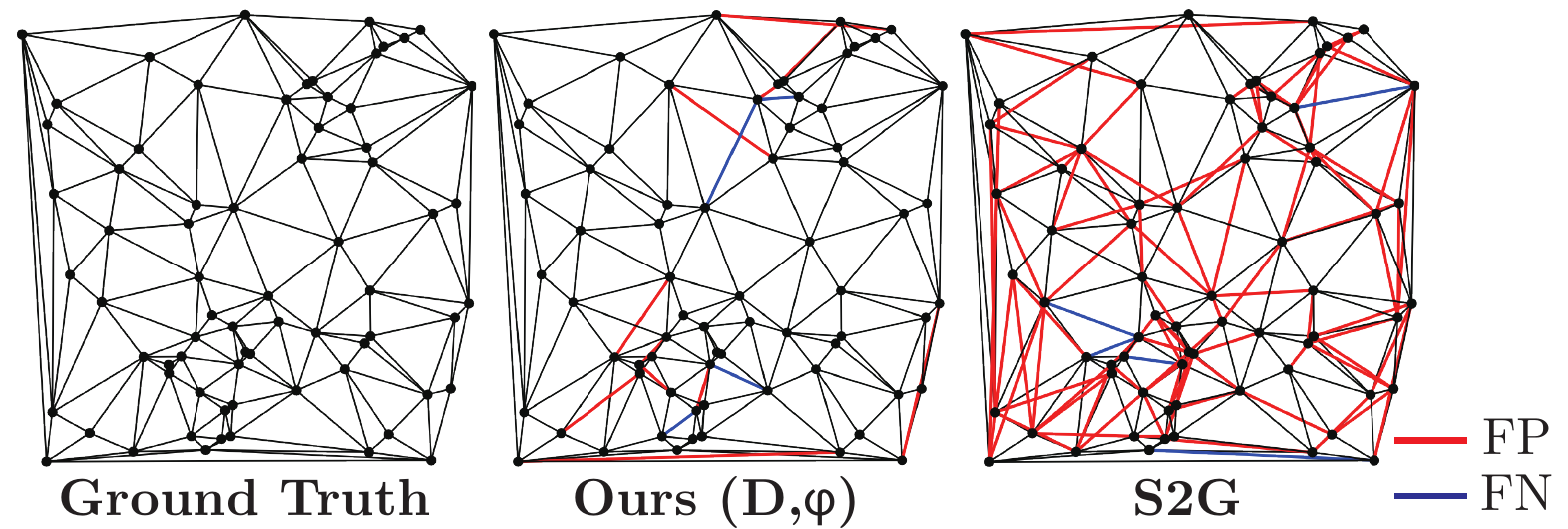}
        \end{minipage}
        }\\
        & MLP & 0.960 & 0.957 & 0.894 \\
        & SIAM & 0.973 & 0.970 & 0.925 \\
        & SIAM-3 & 0.895 & 0.876 & 0.729 \\
        & GNN & 0.972 & 0.970 & 0.929 \\
        & S2G & 0.972 & 0.970 & 0.931\\
        & S2G+ & 0.971 & 0.969 & 0.929 \\
        \cline{2-5}\\[-1em]
        & Ours (D) & \textbf{0.974} & \textbf{0.972} & \textbf{0.935} & \\
        & Ours (D, $\phi$) & \textbf{0.974} & \textbf{0.972} & \textbf{0.935} \\
        \Xhline{2\arrayrulewidth}
    \end{tabular}
\end{adjustbox}
\vspace{-0.2cm}
\end{table}

\vspace{-0.05cm}
The results are outlined in Table~\ref{table:set2graph}.
Mixed-order Transformers, both softmax and kernel attention, have favorable scores over all baselines.
Especially, they outperform all baselines by a large margin in Delaunay: note that GNNs fall into trivial solution with high recall but very low precision.
We particularly note that the Transformers' performance in Delaunay~(20-80) is comparable to Delaunay~(50), with $0.3$-$0.7\%$ drop in F1 score.
Compared with S2G that exhibits $\sim12\%$ drop, this indicates attention mechanism within $\text{Enc}_{1\rightarrow 2}$ is helpful in modeling varying number of nodes.

\cutparagraphup
\paragraph{$k$-uniform hyperedge prediction}
One major advantage of our framework is that it naturally extends to higher-order data (hypergraphs).
To demonstrate this, we consider higher-order extension of set-to-graph prediction task, where the goal is predicting $k$-uniform hyperedges (e.g.,~user-location-activity) from node features.
For evaluation, we used three datasets for transductive 3-edge prediction following the prior work \cite{zhang2020hyper}.
The first dataset GPS derives from a GPS network \cite{zheng2010collaborative}, and contains (user-location-activity) hyperedges.
The second dataset MovieLens is a social network dataset of tagging activities \cite{harper2016the}, containing (user-movie-tag) hyperedges.
The third dataset Drug comes from a medicine network from FAERS\footnote{https://www.fda.gov/Drugs/}, containing (user-drug-reaction) hyperedges.
As baselines, we consider higher-order invariant MLP (S2G+) and the state-of-the-art, self-attention based Hyper-SAGNN \cite{zhang2020hyper}.
We implemented higher-order Transformer and S2G+ by substituting $\text{Enc}_{1\to 2}$ and $L_{1\to 2}$ in set-to-graph architectures to $\text{Enc}_{1\to 3}$ and $L_{1\to 3}$, respectively.
Further details including the datasets, metrics, and baselines can be found in Appendix~\ref{sec:implementation_details_hyperedge} and Appendix~\ref{sec:hyperedge_details}.

The results are in Table~\ref{table:hyperedge_prediction}.
Higher-order Transformer outperforms S2G+ in all datasets, and Hyper-SAGNN in all but one metric in MovieLens.
The results suggest that higher-order self-attention is effective in learning higher-order representation beyond second-order graphs.
The results are encouraging especially because we did not introduce any form of task-specific heuristics into the model, while some of the baselines such as Hyper-SAGNN depend on many inductive biases (static/dynamic branches, Hadamard power, etc.).

\begin{table}[t!]
\caption{$k$-uniform hyperedge prediction results.
For Hyper-SAGNN, we reproduced the scores using the open-sourced code.
For additional baselines including node2vec, we take the scores reported in Zhang~et.~al.~(2020)~\cite{zhang2020hyper}.}
\vspace{0.1cm}
\centering
\begin{adjustbox}{width=0.65\textwidth}
    \label{table:hyperedge_prediction}
    \begin{tabular}{l|cc|cc|cc}
        \Xhline{2\arrayrulewidth}\\[-1em]
        \\[-1em] & \multicolumn{2}{c|}{GPS} & \multicolumn{2}{c|}{MovieLens} & \multicolumn{2}{c}{Drug} \\
        & AUC & AUPR & AUC & AUPR & AUC & AUPR \\
        \\[-1em]\Xhline{2\arrayrulewidth}\\[-1em]
        \\[-1em] node2vec-mean (\cite{zhang2020hyper}) & 0.563 & 0.191 & 0.562 & 0.197 & 0.670 & 0.246 \\
        node2vec-min (\cite{zhang2020hyper}) & 0.570 & 0.185 & 0.539 & 0.186 & 0.684 & 0.258 \\
        DHNE (\cite{zhang2020hyper}) & 0.910 & 0.668 & 0.877 & 0.668 & 0.925 & 0.859 \\
        Hyper-SAGNN-E & 0.947 & 0.788 & 0.922 & \textbf{0.792} & 0.963 & 0.897 \\
        Hyper-SAGNN-W & 0.907 & 0.632 & 0.909 & 0.683 & 0.956 & 0.890 \\
        S2G+ (S) & 0.943 & 0.726 & 0.918 & 0.737 & 0.963 & 0.898 \\
        \\[-1em]\Xhline{2\arrayrulewidth}\\[-1em]
        \\[-1em] Ours (S, $\phi$) & \textbf{0.952} & \textbf{0.804} & \textbf{0.923} & 0.771 & \textbf{0.964} & \textbf{0.901} \\
        [0.1em]\Xhline{2\arrayrulewidth}
    \end{tabular}
\end{adjustbox}
\end{table}

\section{Discussion}\label{sec:discussion}
\cutparagraphup
In this paper, we proposed a generalization of Transformers to higher-orders, and applied a number of design strategies that reduce their complexity to a feasible level.
Higher-order Transformers are attractive, both in theory and application.
In theoretical aspect, it inherits the theoretical completeness and expressive power of invariant MLPs.
In application aspect, it is potentially more powerful than message-passing neural networks due to global interaction modeling, and can be extended to a variety of useful mixed-order tasks involving sets, graphs, and hypergraphs.

At the same time, our work has some limitations that need to be addressed in future work.
First, although complexity to input size can be lowered to linear, the number of basis grows rapidly with increasing order due to $\mathcal{O}((0.792k/\ln(k+1))^k)$ asymptotic formula of $k$-th Bell number \cite{berend2010improved}, still making the model infeasible in higher orders.
Improvement approaches such as finding a compact subset of basis that retains universality \cite{serviansky2020set} and exploiting unorderedness \cite{maron2019invariant} are promising in this direction.
Second, our work builds upon tensor-based representation of graphs, which makes it difficult to be directly extended to hypergraphs containing edges with varying orders (e.g., co-citation networks).
This is a common challenge to all tensor-based methods \cite{hartford2018deep, maron2019invariant, maron2019provably, keriven2019universal, serviansky2020set}, and we believe addressing this would be an important future research direction.

\paragraph{Acknowledgements}
This work was supported in part by National Research Foundation of Korea (NRF) grant funded by the Korea government (MSIT) (2021R1C1C1012540, 2021R1A4A3032834) and Institute of Information \& Communications Technology Planning \& Evaluation (IITP) grant (2021-0-00537, 2019-0-00075).

\newpage
{\small
\bibliography{main}
}

\newpage
\appendix
\section{Appendix}
\cutsubsectionup
\subsection{Proofs}
\subsubsection{Proof of Proposition~\ref{proposition:compactness} (Section~\ref{sec:subsec_higher_order_transformer_layers})}\label{sec:proof_proposition_compactness}
We start with the following lemmas:
\begin{lemma}\label{lemma:class_separation}
Let $\mu$ an equivalence class of order-$(k+l)$ multi-indices.
Then, the set of all $\mathbf{i}\in[n]^k$ such that $(\mathbf{i},\mathbf{j})\in\mu$ for some $\mathbf{j}\in[n]^l$ is an equivalence class of order-$k$ multi-indices.
Likewise, the set of all $\mathbf{j}$ such that $(\mathbf{i},\mathbf{j})\in\mu$ for some $\mathbf{i}$ is an equivalence class of order-$l$ multi-indices.
\end{lemma}
\begin{proof}
We only prove for $\mathbf{i}$, as proof for $\mathbf{j}$ is analogous.
For some $(\mathbf{i}_1, \mathbf{j}_1)\in\mu$, let us denote $\mathbf{i}_1$'s equivalence class as $\mu_k$.
It is sufficient that we prove $\mathbf{i}\in\mu_k\Leftrightarrow(\mathbf{i},\mathbf{j})\in\mu$ for some $\mathbf{j}$.

($\Rightarrow$)
For all $\mathbf{i}\in\mu_k$, as $\mathbf{i}_1\sim\mathbf{i}$ we have $\mathbf{i}=\pi(\mathbf{i}_1)$ for some $\pi\in S_n$.
As $\pi$ acts on multi-indices entry-wise, we have $\pi(\mathbf{i}_1,\mathbf{j}_1)=(\mathbf{i},\pi(\mathbf{j}_1))$.
As the equivalence pattern is invariant to node permutation by definition, we have $\pi(\mathbf{i}_1,\mathbf{j}_1)\sim(\mathbf{i},\pi(\mathbf{j}_1))\sim(\mathbf{i}_1,\mathbf{j}_1)$, and thus $(\mathbf{i},\pi(\mathbf{j}_1))\in\mu$.
Therefore, for all $\mathbf{i}\in\mu_k$, we always have $(\mathbf{i}, \mathbf{j})\in\mu$ when we set $\mathbf{j}=\pi(\mathbf{j}_1)$.

($\Leftarrow$)
For all $(\mathbf{i},\mathbf{j})\in\mu$, as $(\mathbf{i},\mathbf{j})\sim(\mathbf{i}_1,\mathbf{j}_1)$ we have $(\mathbf{i}, \mathbf{j})=\pi(\mathbf{i}_1, \mathbf{j}_1)$ for some $\pi\in S_n$.
We have equivalently $\mathbf{i}=\pi(\mathbf{i}_1)$ and $\mathbf{j}=\pi(\mathbf{j}_1)$ for the $\pi$, which leads to $\mathbf{i}\sim \mathbf{i}_1$ and therefore $\mathbf{i}\in\mu_k$.
\end{proof}
\begin{lemma}\label{lemma:unique_entries}
Let $\mu$ an equivalence class of order-$k$ multi-indices.
Then, every $\mathbf{i}\in\mu$ contains the same number of unique elements, which is equal to $|\mu|$ i.e., the number of nonempty subsets in $\mu$'s partition.
\end{lemma}
\begin{proof}
All $\mathbf{i}\in\mu$ have the same equality pattern, specified by $\mu$'s representative partition.
Specifically, for all $\mathbf{i}\in\mu$, $\mathbf{i}_a=\mathbf{i}_b$ holds iff $\mathbf{i}_a$ and $\mathbf{i}_b$ belong to the same subset within $\mu$'s partition.
Therefore, each nonempty subset within $\mu$'s partition specifies exactly one value within $\mathbf{i}$, and any $\mathbf{i}_a,\mathbf{i}_b$ s.t. $\mathbf{i}_a\neq\mathbf{i}_b$ are contained in distinct subsets within $\mu$'s partition.
Thus, each subset in $\mu$'s partition specifies one unique element in $\mathbf{i}$, and we have the number of unique elements in $\mathbf{i}$ equal to $|\mu|$ for all $\mathbf{i}\in\mu$.
\end{proof}
Now, we prove Proposition~\ref{proposition:compactness}.
\begin{proof}
From Lemma~\ref{lemma:class_separation}, let us denote the set of all $\mathbf{i}\in[n]^k$ such that $(\mathbf{i},\mathbf{j})\in\mu$ as an order-$k$ equivalence class $\mu_k$, and denote the set of all $\mathbf{j}$ such that $(\mathbf{i},\mathbf{j})\in\mu$ as an order-$l$ equivalence class $\mu_q$.

Then, in Eq.~\eqref{eqn:attention_qk_redundant}, to compute $\boldsymbol{\alpha}^\mu_{\mathbf{i},\mathbf{j}}~\forall(\mathbf{i},\mathbf{j})\in\mu$ it is sufficient that we have $\mathbf{K}_\mathbf{i}^\mu ~\forall\mathbf{i}\in\mu_k$ and $\mathbf{Q}_\mathbf{j}^\mu ~\forall\mathbf{j}\in\mu_q$.
Based on the fact, we now analyze and reduce $\mathbf{Q}^\mu = L^{\mu}_{k\to l}(\mathbf{A})$ ($\mathbf{K}^\mu = L^{\mu}_{k\to k}(\mathbf{A})$ can be reduced analogously by letting $l=k$).
From Eq.~\eqref{eqn:equivariant_layer} and Eq.~\eqref{eqn:basis_tensor_weight_bias}, we can write the computation of $\mathbf{Q}^\mu$ as follows, with $\alpha$, $\lambda$ equivalence classes of order-$(k+l)$ and order-$l$ multi-indices and $\mathbf{k}\in[n]^k$:
\begin{align}
    &\mathbf{Q}^\mu_\mathbf{j} = \sum_{\alpha}{\sum_{\mathbf{k}}{\mathbf{B}^{\alpha}_{\mathbf{k}, \mathbf{j}}\mathbf{A}_{\mathbf{k}}w_{\alpha}}} + \sum_{\lambda}{\mathbf{C}^{\lambda}_{\mathbf{j}}b_{\lambda}},\label{eqn:query_redundant}\\
    \text{where }
    &\begin{array}{ll}
        \mathbf{B}_{\mathbf{i}, \mathbf{j}}^{\alpha} = \left\{
        \begin{array}{cc}
            1   &  \text{\scalebox{0.95}{$(\mathbf{k}, \mathbf{j})\in\alpha$}} \\
            0   &  \text{\scalebox{0.95}{ otherwise}}
        \end{array}\right.;&
        \mathbf{C}_{\mathbf{j}}^{\lambda} = \left\{
        \begin{array}{cc}
            1   &  \text{\scalebox{0.95}{$\mathbf{j}\in\lambda$}}\\
            0   &  \text{\scalebox{0.95}{ otherwise}}
        \end{array}\right.\label{eqn:basis_tensor_redundant}
    \end{array}
\end{align}
A key idea is that, when we want $\mathbf{Q}_{\mathbf{j}}^\mu$ only for $\mathbf{j}\in\mu_q$, only a subset of equivalence classes among $\alpha$ or $\lambda$ does effective computation and we can discard the rest.
Specifically, we can discard an equivalence class $\alpha$ if it contains some $(\mathbf{k}, \mathbf{j})$ with $\mathbf{j}\notin\mu_q$.
This is because, for such $\alpha$, $(\mathbf{k}, \mathbf{j})\notin\alpha$ if $\mathbf{j}\in\mu_q$, leading to $\mathbf{B}_{\mathbf{k}, \mathbf{j}}^\alpha=0$ if $\mathbf{j}\in\mu_q$.
Therefore, such $\alpha$ does not contribute to $\mathbf{Q}_\mathbf{j}^\mu ~\forall\mathbf{j}\in\mu_q$ and can be discarded.
On the other hand, an equivalence class $\alpha$ containing some $(\mathbf{k}, \mathbf{j})$ with $\mathbf{j}\in\mu_q$ does effective computation and should be kept.

From that, it turns out that the number of effective $\alpha$ is $\leq\text{b}(k+u_q)$, where $u_q=u(\mathbf{j})=|\mu_q|$ is the number of unique entries within some $\mathbf{j}\in\mu_q$ (see Lemma~\ref{lemma:unique_entries}).
Recall that for an effective $\alpha$, $\mathbf{j}\in\mu_q$ holds for all $(\mathbf{k}, \mathbf{j})\in\alpha$.
Within $\alpha$'s representative partition, as each $\mathbf{j}\in\mu_q$ has exactly $u_q$ unique values, we always have $\{\mathbf{j}_1, ..., \mathbf{j}_l\}$ contained in exactly $u_q$ distinct subsets.
Thus, the possible number of effective $\alpha$ is upper-bounded by the number of ways of partitioning a set with $k+|\mu_q|$ elements, which is $\text{b}(k+u_q)$.
As for the bias, we can repeat the analysis with $k=0$ and the number of effective $\lambda$ is $\leq\text{b}(u_q)$.

We now show that a lower-order linear layer $L^{\mu}_{k\to u_q}$ can compute $\mathbf{Q}^\mu_\mathbf{j}~\forall\mathbf{j}\in\mu_q$ in Eq.~\eqref{eqn:query_redundant}.
Let us denote $\mathcal{A}$ the set of all effective $\alpha$ and $\mathcal{L}$ the set of all effective $\lambda$.
Then we can rewrite Eq.~\eqref{eqn:query_redundant} as:
\begin{align}\label{eqn:q_compact_basis_2}
    \mathbf{Q}^\mu_\mathbf{j} &= \sum_{\alpha\in\mathcal{A}}{\sum_{\mathbf{k}}{\mathbf{B}^{\alpha}_{\mathbf{k}, \mathbf{j}}\mathbf{A}_{\mathbf{k}}w_{\alpha}}} + \sum_{\lambda\in\mathcal{L}}{\mathbf{C}^{\lambda}_{\mathbf{j}}b_{\lambda}},
\end{align}
where $\mathcal{A}$ has $\leq\text{b}(k+u_q)$ elements and $\mathcal{L}$ has $\leq\text{b}(u_q)$ elements.
Assume we have some linear layer $L^{\mu}_{k\to u_q}$.
With $\mathbf{j}'\in[n]^{u_q}$, and $\beta$, $\theta$ equivalence classes of order-$(k+u_q)$ and order-$u_q$ multi-indices respectively, we can write:
\begin{align}
    &\tilde{\mathbf{Q}}^\mu_{\mathbf{j}'} = \sum_{\beta}{\sum_{\mathbf{k}}{\mathbf{B}^{\beta}_{\mathbf{k}, {\mathbf{j}'}}\mathbf{A}_{\mathbf{k}}w_{\beta}}} + \sum_{\theta}{\mathbf{C}^{\theta}_{\mathbf{j}'}b_{\theta}},\label{eqn:query_compact}\\
    \text{where }
    &\begin{array}{ll}
        \mathbf{B}_{\mathbf{k}, \mathbf{j}'}^{\beta} = \left\{
        \begin{array}{cc}
            1   &  \text{\scalebox{0.95}{$(\mathbf{k}, \mathbf{j}')\in\beta$}} \\
            0   &  \text{\scalebox{0.95}{ otherwise}}
        \end{array}\right.;&
        \mathbf{C}_{\mathbf{j}'}^{\theta} = \left\{
        \begin{array}{cc}
            1   &  \text{\scalebox{0.95}{$\mathbf{j}'\in\theta$}}\\
            0   &  \text{\scalebox{0.95}{ otherwise}}
        \end{array}\right.\label{eqn:basis_tensor_compact}
    \end{array}
\end{align}

We now identify the condition that $\tilde{\mathbf{Q}}^\mu$ contains all $\mathbf{Q}^\mu_\mathbf{j}~\forall\mathbf{j}\in\mu_q$.
For that, we need to define a mapping between index space of $\tilde{\mathbf{Q}}^\mu$ and $\mathbf{Q}^\mu$.
To this end, we define a surjection $g: [k]\to[u_q]$ that satisfies $\mathbf{j}_{a}=\mathbf{j}_{b}\Leftrightarrow g(a)=g(b)$.
We can always define such $g$ due to the property of equivalence classes that, for all $a, b\in[l]$, $\mathbf{j}_{a}=\mathbf{j}_{b}$ holds iff $\mathbf{j}_{a}, \mathbf{j}_{b}$ belong to a same subset within $\mu_q$'s partition.
By indexing the subsets within $\mu_q$'s partition, we define $g(a)~\forall a\in[l]$ as the index of the subset that $a$ belongs.
Then, for every $\mathbf{j}\in\mu_q$, we can find an order-$u_q$ compact form $\mathbf{j}'\in[n]^{u_q}$ containing $u_q$ unique elements through $g$: for $c\in[u_q]$ that $c=g(a)=g(b)=\cdots$,  we construct $\mathbf{j}'$ such that $\mathbf{j}'_c=\mathbf{j}_a=\mathbf{j}_b=\cdots$.
We define $f_\mu^q:[n]^q\to[n]^{u_q}$ as a mapping that gives $\mathbf{j}'=f_\mu^q(\mathbf{j})~\forall\mathbf{j}\in\mu_q$.

We now reduce Eq.~\eqref{eqn:q_compact_basis_2} into Eq.~\eqref{eqn:query_compact}.
First, for each $\alpha\in\mathcal{A}$, we assign a distinct order-$(k+u_q)$ equivalence class $\beta$ that satisfies: $(\mathbf{k},\mathbf{j})\in\alpha\Leftrightarrow(\mathbf{k},\mathbf{j}')\in\beta$ with $\mathbf{j}'=f_\mu^q(\mathbf{j})$.
This can be done by changing $\alpha$'s partition into $\beta$'s, by merging each set of $\mathbf{j}_a=\mathbf{j}_b=...$ into corresponding $\mathbf{j}'_c$ following $f_\mu^q$.
We similarly assign a distinct $\theta$ to each $\lambda\in\mathcal{L}$.
Then, we set $w_\alpha=w_\beta$ for all paired $\alpha$ and $\beta$, and set $w_\beta=0$ for every $\beta$ not paired with any $\alpha$.
We similarly set $b_\theta=b_\lambda$ for all paired $\theta$ and $\lambda$, and $b_\theta=0$ for every $\theta$ not paired with any $\lambda$.
From the definition of basis tensors, $\mathbf{B}_{\mathbf{k},\mathbf{j}}^\alpha=\mathbf{B}_{\mathbf{k},\mathbf{j}'}^\beta$ for all paired $\alpha$ and $\beta$, and $\mathbf{C}_{\mathbf{j}}^\lambda=\mathbf{C}_{\mathbf{j}'}^\theta$ for all paired $\lambda$ and $\theta$.
Therefore, we have $\tilde{\mathbf{Q}}^\mu_{\mathbf{j}'}=\mathbf{Q}^\mu_{\mathbf{j}}$ for all $\mathbf{j}\in\mu_q$.
Conclusively, we can always construct $L^{\mu}_{k\to u_q}$ and $f_q^\mu: [n]^{l}\to[n]^{u_q}$ that gives $\mathbf{Q}_\mathbf{j}^\mu~\forall\mathbf{j}\in\mu_q$.

As noted in the beginning of the proof, we can perform the same analysis with $k=l$ to show the analogous result for $\mathbf{K}^\mu$.
We now have all entries $\mathbf{K}_\mathbf{i}^\mu~\forall\mathbf{i}\in\mu_k$ and $\mathbf{Q}_\mathbf{j}^\mu~\forall\mathbf{j}\in\mu_q$ to compute $\boldsymbol{\alpha}^\mu_{\mathbf{i},\mathbf{j}}~\forall(\mathbf{i},\mathbf{j})\in\mu$ (Eq.~\eqref{eqn:attention_qk_redundant}) and therefore Proposition~\ref{proposition:compactness} holds.
\end{proof}
\begin{figure}[!t]
    \centering
    \includegraphics[width=0.8\textwidth]{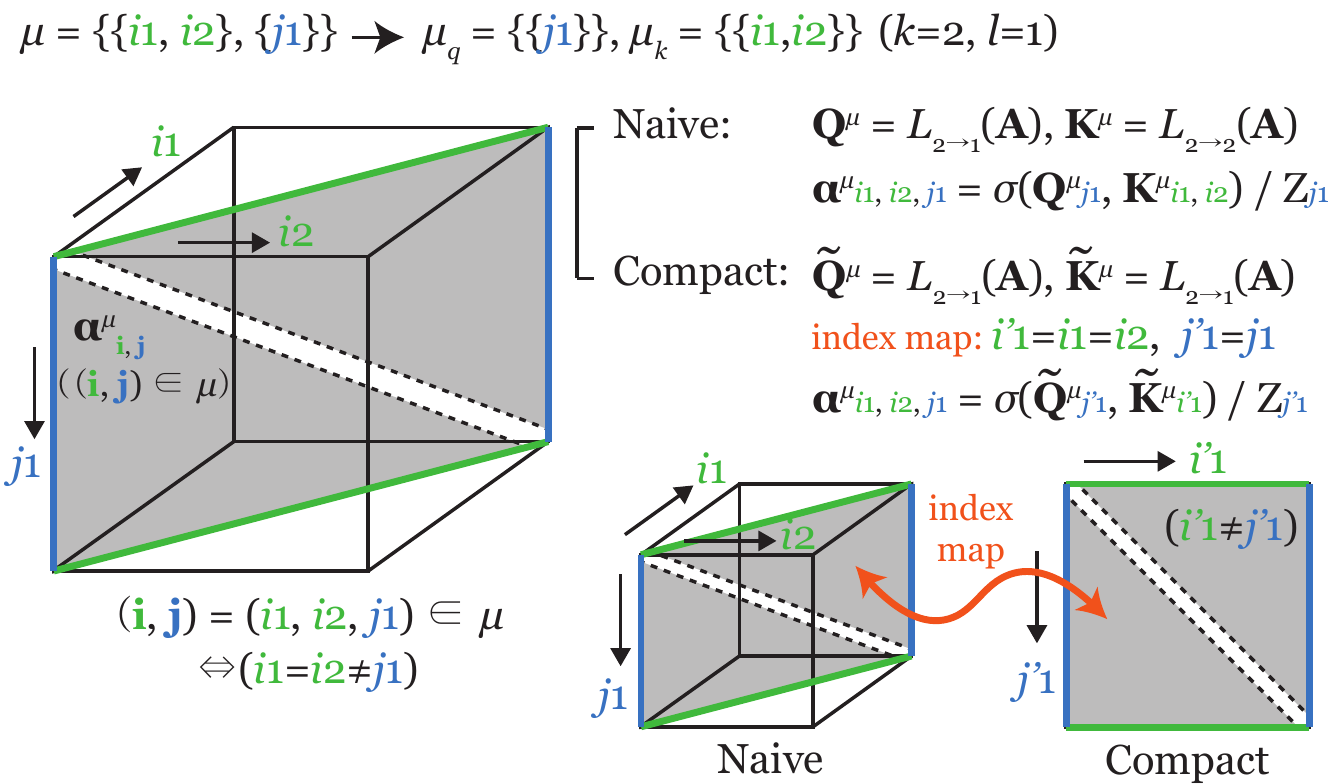}
    \vspace{-0.2cm}
    \caption{
    Exemplar illustration of computing $\boldsymbol{\alpha}^\mu_{\mathbf{i},\mathbf{j}}~\forall(\mathbf{i},\mathbf{j})\in\mu$ with lower-order query and key, for $k=2$, $l=1$, $\mathbf{i}=(i_1,i_2)$, $\mathbf{j}=(j_1)$, and $\mu=\{\{i_1, i_2\}, \{j_1\}\}$.
    }
    \label{fig:compact_illustration}
    \vspace{-0.3cm}
\end{figure}
In Figure~\ref{fig:compact_illustration} we provide an example of computing $\boldsymbol{\alpha}^\mu$ with lower-order query and key.

\subsubsection{Proof of Property~\ref{property:complexity} (Section~\ref{sec:implementation})}\label{sec:proof_property_complexity}
\begin{proof}
We begin by analyzing the complexity of an equivariant linear layer $L_{k\to l}$ (Eq.~\eqref{eqn:equivariant_layer}).
In the inner summation $\sum_\mathbf{i}\mathbf{B}_{\mathbf{i},\mathbf{j}}^{\mu}\mathbf{A}_\mathbf{i}$,
with $u_k$ and $u_q$ the number of unique entries in $\mathbf{i}$ and $\mathbf{j}$ respectively, we have $n^{\leq u_k}$ effective multiplication and summations for each output index $\mathbf{j}$.
Inequality is when $\mathbf{j}_b=\mathbf{i}_a$ for some $a$, $b$, which corresponds to indexing operation rather than summation.
Thus, the number of operations done by the summation is $n^{u_q}n^{\leq u_k}$.
With outer summation over $\mu$, we have $\sum_\mu{n^{u_q}n^{\leq u_k}}$ operations.
As $u_q\leq l$ and $u_k\leq k$ by definition, we have inequality $\sum_\mu{n^{u_q}n^{\leq u_k}} \leq \text{b}(k+l)n^{k+l}$.
Application of $w_\mu$ gives us $dd'\sum_\mu{n^{u_q}n^{u_k}} \leq \text{b}(k+l)dd'n^{k+l}$ number of operations.
For the bias, in the inner term $\mathbf{C}_\mathbf{j}^\lambda$, we need a single addition for each $\mathbf{j}$ and thus the number of operations for a $\lambda$ is $n^{u_q}\leq n^l$.
Summation over $\lambda$ and application of bias parameters gives us $\text{b}(l)d'n^{u_q}\leq \text{b}(l)d'n^l$ operations.
Collectively, we have $\leq \text{b}(k+l)dd'n^{k+l} + \text{b}(l)d'n^l$ number of operations.
As $k, l, d, d'$ are constants that does not depend on $n$, we obtain $\mathcal{O}(n^{k+l})$ complexity.

Computation of $\text{Enc}_{k\to l}(\mathbf{A})$ (Eq.~\eqref{eqn:higher_transformer_layer}) involves computing 
$\text{Attn}_{k\to l}(\mathbf{A})$, $\text{MLP}_{l\to l} (\text{Attn}_{k\to l}(\mathbf{A}))$ and adding them.
Let us analyze $\text{Attn}_{k\to l}(\mathbf{A})$ first.
To compute $\boldsymbol{\alpha}^{h,\mu}$ from input, we need to compute $L^{\mu}_{k\to u_q}(\mathbf{A})$ and $L^{\mu}_{k\to u_k}(\mathbf{A})$, followed by pairwise similarity computation and re-indexing.
Assuming that each pairwise similarity computation and indexing has constant complexity, we have $\mathcal{O}(n^{k+u_q}+n^{k+u_k}+n^{u_q+u_k})$.\footnote{Normalization over keys gives an additive complexity $\mathcal{O}(n^{u_q+u_k})$, which can be absorbed to the formula.}
As $u_q\leq l$ and $u_k\leq k$, we have $\mathcal{O}(n^{k+l}+n^{2k})$.
It is worth to note that $\mathcal{O}(n^{2k})$ term comes from computation of keys from input.
Having computed $\boldsymbol{\alpha}^{h,\mu}$, the inner summation $\sum_\mathbf{i}\boldsymbol{\alpha}_{\mathbf{i},\mathbf{j}}^{h,\mu}\mathbf{A}_\mathbf{i}$, similar to in $L_{k\to l}$, has $n^{u_q}n^{\leq u_k}$ computations.
With outer summation over $\mu$, we have $\sum_\mu{n^{u_q}n^{\leq u_k}}\leq \text{b}(k+l)n^{k+l}$ operations.
Summation over heads and application of weight matrices gives us $\leq \text{b}(k+l)Hd_H^2dn^{k+l}$ operations, which is $\mathcal{O}(n^{k+l})$.
For application of $\text{MLP}_{l\to l}(\cdot)$, we sum the complexity of two linear layers $L_{l\to l}$ and element-wise $\text{ReLU}$, which gives us $\mathcal{O}(n^{2l})$.
Conclusively, the complexity of $\text{Enc}_{k\to l}$ is $\mathcal{O}(n^{2k}+n^{k+l}+n^{2l})$.
\end{proof}

\subsubsection{Proof of Proposition~\ref{proposition:lightweight} (Section~\ref{sec:lightweight})}\label{sec:proof_proposition_lightweight}
\begin{proof}
We assume $k, l > 0$.
Among the equivalence classes $\mu$ of order-$(k+l)$ multi-indices, let us select a subset $\mathcal{M}$ that all $\mu\in\mathcal{M}$ satisfies the following: for all $(\mathbf{i},\mathbf{j})\in\mu$, $\mathbf{i}_a=\mathbf{j}_b$ holds for all $a\in[k]$ and some $b\in[l]$.
In other words, every element in $\mathbf{i}$ is identical with at least one element in $\mathbf{j}$, and $\mathbf{i}$ becomes a single fixed multi-index when we fix $\mathbf{j}$ (we denote the fixed $\mathbf{i}=\text{fix}(\mathbf{j})$).
This renders $\mathbf{B}_{\mathbf{i},\mathbf{j}}^\mu=1\Leftrightarrow\mathbf{i}=\text{fix}(\mathbf{j})$ for such $\mu$, and consequently the inner-summation $\sum_{\mathbf{i}}{\mathbf{B}^{\mu}_{\mathbf{i}, \mathbf{j}}\mathbf{A}_{\mathbf{i}}w_{\mu}}$ in Eq.~\eqref{eqn:lightweight} reduces to elementwise indexing $\mathbf{A}_{\text{fix}(\mathbf{j})}w_{\mu}$.
As the size of $\mathcal{M}$ is upper-bounded by a constant $\text{b}(k+l)$, we have $\mathcal{O}(n^l)$ complexity when computing Eq.~\eqref{eqn:lightweight}.
With the trivial case $\mu=\{\{i_1, ..., i_k, j_1, ..., j_l\}\}\in\mathcal{M}$, we can always find nonempty $\mathcal{M}$.
\end{proof}
\begin{figure}[!t]
    \centering
    \vspace{-0.4cm}
    \includegraphics[width=0.75\textwidth]{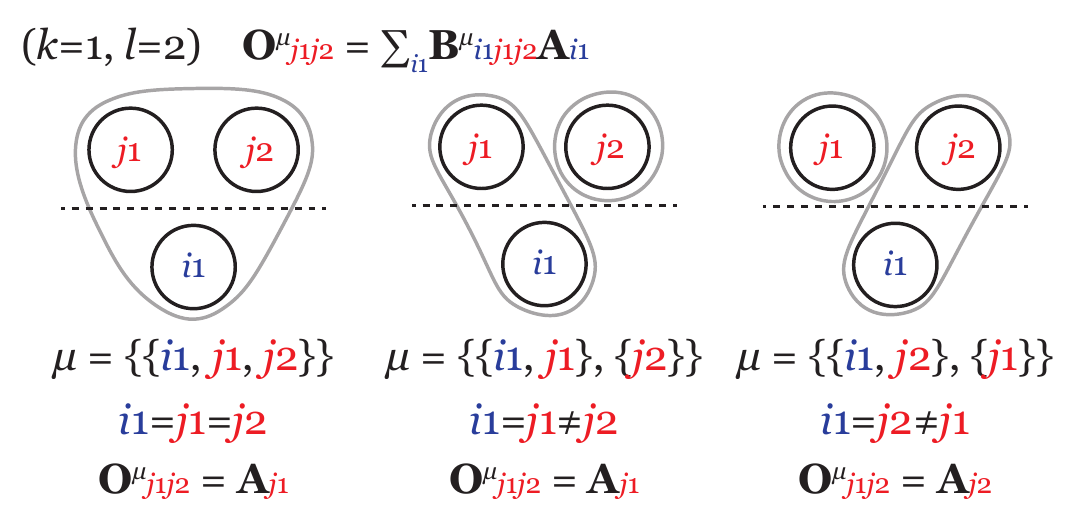}
    \vspace{-0.2cm}
    \caption{
    Exemplar illustration of all equivalence classes included in lightweight linear layer $\bar{L}_{1\to 2}$.
    }
    \label{fig:lightweight_mu}
\end{figure}
To provide some intuition, we illustrate all $\mu\in\mathcal{M}$ for $k=1$, $l=2$ in Figure~\ref{fig:lightweight_mu}.

\subsubsection{Validity of Proposition~\ref{thm:generalization} when using lightweight linear layers (Section~\ref{sec:lightweight})}\label{sec:proof_validity_lightweight}
As stated in the main text, $\text{Enc}_{k\to l}$ (Eq.~\eqref{eqn:higher_transformer_layer}) with linear layers for key, query, and MLP changed to $\bar{L}$ still generalizes $L_{k\to l}$.
This can be shown simply by plugging $\bar{L}$ into the proof of Proposition~\ref{thm:generalization}.
We can still assume $\boldsymbol{\alpha}^{h,\mu}_{\mathbf{i},\mathbf{j}}=1$ for all \scalebox{0.95}{$(\mathbf{i},\mathbf{j})\in\mu$} by setting $\bar{L}$ for key and query to output constants, and can reduce $\text{MLP}_{l\to l}$ composed of $\bar{L}$ to an invariant bias as we subsample $\mu\in\mathcal{M}$ but keep all $\lambda$ for the bias.
Thus, Eq.~\eqref{eqn:higher_transformer_layer} can reduce to Eq.~\eqref{eqn:equivariant_layer} and Proposition~\ref{thm:generalization} holds.

\subsubsection{Proof of Property~\ref{property:sparse_complexity} (Section~\ref{sec:sparse})}\label{sec:proof_property_sparse_complexity}
\begin{proof}
We begin from sparse equivariant linear layer $L_{k\to l}$ (Eq.~\eqref{eqn:sparse_equivariant_layer}).
In the inner summation $\sum_{\mathbf{i}\in E}\mathbf{B}_{\mathbf{i},\mathbf{j}}^{\mu}\mathbf{A}_\mathbf{i}$,
the number of multiplication and addition for each $\mathbf{j}$ is upper-bounded by $m=|E|$.
As the number of output multi-indices $\mathbf{j}$ is bounded by $|E'|\leq m\binom{k}{l}$, the effective number of operations are $\leq m^2\binom{k}{l}$.
With outer summation over $\mu$, we have $\leq \text{b}(k+l)\binom{k}{l}m^2$ operations, leading to  complexity $\mathcal{O}(m^2)$.
For the lightweight linear layers $\bar{L}$ (Proposition~\ref{proposition:lightweight}) that precludes summation over input, we trivially have $\mathcal{O}(m)$ complexity as we do not sum over $\mathbf{i}$.

Now, we analyze the complexity of sparse self-attention computation (Eq.~\eqref{eqn:sparse_self_attention}).
To compute $\boldsymbol{\alpha}^{\mu}$ from input, we need to compute lightweight linear layers $\bar{L}^{\mu}_{k\to u_q}(\mathbf{A}, E)$ and $\bar{L}^{\mu}_{k\to u_k}(\mathbf{A}, E)$, followed by pairwise similarity computation of nonzero entries.
As $u_q, u_k \leq k$, we have complexity $\mathcal{O}(m)$ for the linear layers and $\mathcal{O}(m^2)$ for pairwise computation.
Having computed $\boldsymbol{\alpha}^{\mu}$, the inner summation $\sum_{\mathbf{i}\in E} \boldsymbol{\alpha}_{\mathbf{i},\mathbf{j}}^{\mu}\mathbf{A}_\mathbf{i}$ has $\leq m$ computations.
Enumerating over $\mathbf{j}$, we have $\mathcal{O}(m^2)$.

Finally, we analyze the complexity of $\text{Enc}_{k\to l}$ composed of the sparse linear layers and self-attention.
This involves adding the outputs of 
$\text{Attn}_{k\to l}(\mathbf{A}, E)$ (which we already addressed) and $\text{MLP}_{l\to l}(\text{Attn}_{k\to l}(\mathbf{A}, E), E)$.
For application of $\text{MLP}_{l\to l}$, we sum the complexity of two lightweight linear layers $\bar{L}_{l\to l}$ and element-wise $\text{ReLU}$, which gives us $\mathcal{O}(m)$.
In summary, the complexity of sparse $\text{Enc}_{k\to l}$ is $\mathcal{O}(m^2)$.
\end{proof}

\subsubsection{Proof of Property~\ref{property:kernel_complexity} (Section~\ref{sec:kernel_linear})}\label{sec:proof_property_kernel_complexity}
\begin{proof}
Summation over $\mathbf{i}\in\mathcal{I}$ decouples $\mathbf{i}$ from $\mathbf{j}$ and allows reuse of computation over $\mathbf{j}$.
As the summation over $\mathbf{i}\in\mathcal{I}$ involves $\mathcal{O}(n^{k})$ operations and we share it over all query indices $\mathbf{j}$, self-attention reduces to elementwise application and we obtain $\mathcal{O}(n^k+n^l)$ complexity.
As computation of $\tilde{\mathbf{Q}}^{\mu}$, $\tilde{\mathbf{K}}^{\mu}$ and application of $\text{MLP}_{l\to l}$ are $\mathcal{O}(n^k+n^l)$ with lightweight linear layers, we have $\mathcal{O}(n^k+n^l)$ collective complexity for $\text{Enc}_{k\to l}$.
When adopted into sparse $\text{Enc}_{k\to l}$ (Eq.~\eqref{eqn:sparse_self_attention}), summation over $\mathbf{i}$ and enumeration over $\mathbf{j}$ all reduce to $\mathcal{O}(m)$ and we thereby have $\mathcal{O}(m)$ complexity.
\end{proof}

\subsubsection{Proof of Theorem~\ref{thm:message_passing} (Section~\ref{sec:message_passing})}\label{sec:proof_theorem_message_passing}
\begin{proof}
With message function $M:\mathbb{R}^{2d_v+d_e}\to \mathbb{R}^{d_m}$ and update function $U:\mathbb{R}^{d_v+d_m}\to \mathbb{R}^{d}$, a message passing step takes node features $\mathbf{X}\in\mathbb{R}^{n\times d_v}$ and edge features $\mathbf{E}\in\mathbb{R}^{n\times n\times d_e}$ as input and outputs node features $\mathbf{H}\in\mathbb{R}^{n\times d}$ according to following.
With $i,j\in[n]$:
\begin{align}
    \mathbf{M}_j &= \sum_{i\in\mathcal{N}(j)}{M(\mathbf{X}_j, \mathbf{X}_i, \mathbf{E}_{ij})}\\
    \mathbf{H}_j &= U(\mathbf{X}_j, \mathbf{M}_j),
\end{align}
where $\mathcal{N}(j)$ denotes incoming neighbors of $j$-th node, \emph{i.e.,} $\{i | (i,j)\in E\}$.

We now show how a composition of two $\text{Enc}_{2\to 2}$ can approximate above computation.
\begin{enumerate}
    \item As a first step, we encode $\mathbf{X}$ and $\mathbf{E}$ into a single $\mathbf{A}\in\mathbb{R}^{n\times n\times (2d_v+d_e)}$ \cite{maron2019invariant}.
    In the first $d_v$ channels, we replicate $\mathbf{X}$ on the rows.
    In the next $d_v$ channels, we replicate $\mathbf{X}$ on the columns.
    In the last $d_e$ channels, we put $\mathbf{E}$.
    Additionally, to account for output positions (node features), we augment $E$ with self-loops and make $E'=E\cup\{(i,i)~\forall i\in[n]\}$.
    \item Then, we make the first $\text{Enc}_{2\to 2}$ approximate the message function $M(\cdot)$, so that $\text{Enc}_{2\to 2}(\mathbf{A})_{ij}\approx M(\mathbf{X}_j, \mathbf{X}_i, \mathbf{E}_{ij})$.
    To do this, we first reduce $\text{Attn}_{2\to 2}(\mathbf{A})_{ij}=\mathbf{A}_{ij}$ and apply $\text{MLP}_{l\to l}$ on top of it.
    We reduce $\text{MLP}_{l\to l}$ to entry-wise $\text{MLP}$.
    As $\text{Attn}_{2\to 2}(\mathbf{A})_{ij}=\mathbf{A}_{ij}$ is a concatenation of $\mathbf{X}_i$, $\mathbf{X}_j$, $\mathbf{E}_{ij}$, with universal approximation theorem \cite{hornik1989multilayer}, we can have the output of the first $\text{Enc}_{2\to 2}(\mathbf{A})_{ij}=M(\mathbf{X}_j, \mathbf{X}_i, \mathbf{E}_{ij})+\epsilon_1~\forall i, j$ where $\epsilon_1$ is approximation error.
    This also holds when we leverage sparsity and restrict the index scope to $(i, j)\in E'$; in this case, we make $(i, j)\in E'\setminus E$ contain zero vectors.
    \item Before feeding the output to the second $\text{Enc}_{2\to 2}$, we concatenate the original input $\mathbf{A}$ with the output of the first layer in the channel dimension to make $\mathbf{A}'\in\mathbb{R}^{n\times n\times{(2d_v+d_e+d_m)}}$.
    This gives $\mathbf{X}_i, \mathbf{X}_j, \mathbf{E}_{ij}$ encoded in the first $2d_v+d_e$ channels and $M(\mathbf{X}_j, \mathbf{X}_i, \mathbf{E}_{ij})+\epsilon_1$ encoded in the last $d_m$ channels of $\mathbf{A}'$.
    This operation can be trivially absorbed within the MLP of the first layer, but we separate for simplicity.
    \item Now, we make the second $\text{Enc}_{2\to 2}$ jointly approximate summation of messages over neighbors $\sum_{i\in\mathcal{N}(j)}(\cdot)$ and update function $U(\cdot)$, so that $\text{Enc}_{2\to 2}(\mathbf{A}')_{jj}\approx \mathbf{H}_j = U(\mathbf{X}_j,\mathbf{M}_j)$.
    First, we reduce $\text{Attn}_{2\to 2}(\mathbf{A}')$ to summation over neighbors.
    For this we only need two equivalence classes $\mu_1=\{\{1\}, \{2, 3, 4\}\}$ and $\mu_2=\{\{1, 2, 3, 4\}\}$.
    Omitting normalization, we can write Eq.~\eqref{eqn:kernel_querywise} as follows.
    For $\mu_1$ we set $u_q=1$, $u_k=2$, $\mathbf{i}=(i,j)$, $\mathbf{j}=(j,j)$, $\mathbf{i}'=(i,j)$, $\mathbf{j}'=j$, and for $\mu_2$ we set $u_q=1$, $u_k=1$, $\mathbf{i}=(j,j)$, $\mathbf{j}=(j,j)$, $\mathbf{i}'=j$, $\mathbf{j}'=j$.
    \begin{align}\label{eqn:reduced_2_1_attn}
        \text{Attn}_{2\to 2}(\mathbf{A}')_{jj} = \phi(\tilde{\mathbf{Q}}^{\mu_1}_{j})^\top \sum_{\{i|(i,j)\in E'\}}{\phi(\tilde{\mathbf{K}}^{\mu_1}_{ij})}
        \mathbf{A}'_{ij}w_{\mu_1} + \phi(\tilde{\mathbf{Q}}^{\mu_2}_{j})^\top \phi(\tilde{\mathbf{K}}^{\mu_2}_{j})
        \mathbf{A}'_{jj}w_{\mu_2}.
    \end{align}
    Let entries in $\phi(\tilde{\mathbf{K}}^{\mu})$, $\phi(\tilde{\mathbf{Q}}^{\mu})$ be $\frac{1_{d_K}}{\sqrt{d_K}}$ so that their dot product is 1.
    Eq.~\eqref{eqn:reduced_2_1_attn} reduces to:
    \begin{align}
        \text{Attn}_{2\to 2}(\mathbf{A}')_{jj} &= \sum_{\{i|(i,j)\in E'\}}{\mathbf{A}'_{ij}w_{\mu_1}} + \mathbf{A}'_{jj}w_{\mu_2}
    \end{align}
    We make $w_{\mu_1}$ zero-out the first $2d_v+d_e$ channels and $w_{\mu_2}$ zero-out the last $d_e+d_m$ channels.
    Then, we have $\text{Attn}_{2\to 2}(\mathbf{A}')_{jj}$ contain $\mathbf{X}_j$ in the first $d_v$ channels and $\mathbf{M}_j + \epsilon_2=\sum_{i | (i,j)\in E'}(M(\mathbf{X}_j, \mathbf{X}_i, \mathbf{E}_{ij}))+\epsilon_2$ in the last $d_m$ channels where $\epsilon_2$ is approximation error\footnote{Note that message summation over $\{i | (i,j)\in E'\}$ is equivalent to summation over $\{i | (i,j)\in E\}=\mathcal{N}(j)$ because we set message zero at $(i,j)\in E'\setminus E$.}.
    We then apply $\text{MLP}_{l\to l}$ on top of it, which can approximate the update function by universal approximation theorem \cite{hornik1989multilayer} and we have $\text{Attn}_{2\to 2}(\mathbf{A}')_{jj} = U(\mathbf{X}_j, \mathbf{M}_j) + \epsilon_3$ where $\epsilon_3$ is approximation error.
\end{enumerate}
Overall, the approximation error $\epsilon_i$ at each step depends on $\epsilon_{i-1}$ $(i>1)$, the MLP that approximates relevant function, and uniform bounds and uniform continuity of the approximated functions \cite{hornik1989multilayer}.

In the opposite, message passing cannot approximate some of the operations done by a single Transformer layer $\text{Enc}_{2\to 2}$.
This can be seen from the fact that, given a graph with diameter $d(E)$, we need at least $d(E)$ message passing operations to approximate output of $\text{Enc}_{2\to 2}$.
This is because a single $\text{Enc}_{2\to 2}$ can impose dependency between any pair of input and output indices $\mathbf{i}, \mathbf{j}$, while message passing requires $d(E)$ steps in the worst case.
Consequently, the approximation becomes impossible when the graph contains $>1$ disconnected components, which leads to $d(E)\to+\infty$.
\end{proof}

\subsection{Experimental details (Section~\ref{sec:experiments})}\label{sec:experimental_details}
\begin{table}[!t]
\caption{Statistics of the datasets.}
\label{table:dataset_statistics}
\begin{subtable}[h]{.5\linewidth}
    \caption{Statistics of the synthetic chains dataset.}
    \centering
    \label{table:chain_dataset}
        \begin{tabular}{ll}
        \Xhline{2\arrayrulewidth}\\[-1em]
        \\[-1em] Dataset & Chains \\
        \\[-1em]\Xhline{2\arrayrulewidth}\\[-1em]
        \\[-1em] Size & 60 \\
        \\[-1em] \# classes & 2 \\
        \\[-1em] Average \# node & 20 (train) / 200 (test) \\
        \\[-1em]\Xhline{2\arrayrulewidth}
    \end{tabular}
\end{subtable}\hfill
\begin{subtable}[h]{.5\linewidth}
    \caption{Statistics of the PCQM4M-LSC dataset.}
    \centering
    \label{table:regression_dataset}
        \begin{tabular}{ll}
        \Xhline{2\arrayrulewidth}\\[-1em]
        \\[-1em] Dataset & PCQM4M-LSC \\
        \\[-1em]\Xhline{2\arrayrulewidth}\\[-1em]
        \\[-1em] Size & 3.8M \\
        \\[-1em] Average \# node & 14.1 \\
        \\[-1em] Average \# edge & 14.6 \\
        \\[-1em]\Xhline{2\arrayrulewidth}
    \end{tabular}
\end{subtable}\vfill
\vspace{0.2cm}
\centering
\begin{subtable}[h]{.7\linewidth}
    \caption{Dataset statistics for set-to-graph prediction.}
    \centering
    \label{table:set2graph_dataset}
    \begin{tabular}{llll}
        \Xhline{2\arrayrulewidth}\\[-1em]
        \\[-1em] Dataset & Jets & Delaunay (50) & Delaunay (20-80) \\
        \\[-1em]\Xhline{2\arrayrulewidth}\\[-1em]
        \\[-1em] Size & 0.9M & 55k & 55k \\
        \\[-1em] Average \# node & 7.11 & 50 & 50 \\
        \\[-1em] Average \# edge & 35.9 & 273.6 & 273.9 \\
        \\[-1em]\Xhline{2\arrayrulewidth}
    \end{tabular}
\end{subtable}\vfill
\vspace{0.2cm}
\centering
\begin{subtable}[h]{.7\linewidth}
    \caption{Dataset statistics for $k$-uniform hyperedge prediction. For each dataset, each row under "\#~nodes" correspond to each row under "Node~types".}
    \centering
    \label{table:hyperedge_datasets}
    \begin{tabular}{llll}
        \Xhline{2\arrayrulewidth}\\[-1em]
        \\[-1em] Dataset & GPS & MovieLens & Drug \\
        \\[-1em]\Xhline{2\arrayrulewidth}\\[-1em]
        \\[-1em] \multirow{3}{5em}{Node types} & user & user & user \\
        \\[-1em] & location & movie & drug \\
        \\[-1em] & activity & tag & reaction \\[-1em]\\
        \\[-1em]\Xhline{2\arrayrulewidth}\\[-1em]
        \\[-1em] \multirow{3}{5em}{\# nodes} & 146 & 2,113 & 12 \\
        \\[-1em] & 70 & 5,908 & 1,076 \\
        \\[-1em] & 5 & 9,079 & 6,398 \\[-1em]\\
        \\[-1em]\Xhline{2\arrayrulewidth}\\[-1em]
        \\[-1em] \# edges & 1,436 & 47,957 & 171,756 \\[-1em]\\
        \\[-1em]\Xhline{2\arrayrulewidth}
    \end{tabular}
\end{subtable}
\end{table}

\begin{table}[!p]
\caption{Architectures of the models used in our experiments.
$\text{Enc}_{k\to l}(d, d_H, H)$ denotes $\text{Enc}_{k\to l}$ with hidden dimension $d$, head dimension $d_H$, and number of heads $H$.
$L_{k\to l}(d)$ denotes $L_{k\to l}$ with output dimension $d$.
$\text{MLP}(n, d, d_\text{out})$ denotes an elementwise MLP with $n$ hidden layers, hidden dimension $d$, output dimension $d_\text{out}$, and ReLU non-linearity.}
\label{table:architectures}
\begin{subtable}{.99\linewidth}
\vspace{-0.1cm}
\caption{Architectures for runtime and memory analysis.}
\vspace{-0.1cm}
\centering
\label{table:resource_architecture}
\begin{tabular}{lc}
    \Xhline{2\arrayrulewidth}\\[-1em]
    \\[-1em] Method & Architecture \\
    \\[-1em]\Xhline{2\arrayrulewidth}\\[-1em]
    \\[-1em] $\text{MLP}_\pi$ (D/S) & \cameraready{$[L_{2\to 2}(32)-\text{ReLU}]_{\times4}$-$L_{2\to 0}(32)$} \\
    \\[-1em] Ours (D/S, ($\phi$)) & \cameraready{$\text{Enc}_{2\to 2,\phi}(32, 8, 4)_{\times4}$-$\text{Enc}_{2\to 0}(32, 8, 4)$-LN-Linear(32)} \\
    \\[-1em]\Xhline{2\arrayrulewidth}
\end{tabular}
\end{subtable}\vfill
\vspace{0.4cm}
\begin{subtable}{.99\linewidth}
\caption{Architectures for chain experiment. Output dimension of a layer is denoted in parenthesis.}
\centering
\label{table:chain_architecture}
\begin{tabular}{lc}
    \Xhline{2\arrayrulewidth}\\[-1em]
    \\[-1em] Method & Architecture \\
    \\[-1em]\Xhline{2\arrayrulewidth}\\[-1em]
    \\[-1em] GCN & \cameraready{GCNConv(16)-ReLU-GCNConv(16)-Linear(2)} \\
    \\[-1em] GIN-0 & \cameraready{GINConv(16)-ReLU-GINConv(16)-Linear(2)} \\
    \\[-1em] GAT & \cameraready{GATConv(16)-ReLU-GATConv(16)-Linear(2)} \\
    \\[-1em]\Xhline{2\arrayrulewidth}\\[-1em]
    \\[-1em] $\text{MLP}_\pi$ (S) & \cameraready{$L_{2\to 2}$(16)-ReLU-$L_{2\to 1}$(16)-Linear(2)} \\
    \\[-1em] Ours (S) w/o global & \cameraready{$\text{Enc}_{2\to 2,\text{ablated}}$(16)-$\text{Enc}_{2\to 1,\text{ablated}}$(16)-LN-Linear(2)} \\
    \\[-1em] Ours (S, $\phi$) w/o global & \cameraready{$\text{Enc}_{2\to 2,\phi,\text{ablated}}$(16)-$\text{Enc}_{2\to 1,\phi,\text{ablated}}$(16)-LN-Linear(2)} \\
    \\[-1em]\Xhline{2\arrayrulewidth}\\[-1em]
    \\[-1em] Ours (S) & \cameraready{$\text{Enc}_{2\to 2}$(16)-$\text{Enc}_{2\to 1}$(16)-LN-Linear(2)} \\
    \\[-1em] Ours (S, $\phi$) & \cameraready{$\text{Enc}_{2\to 2,\phi}$(16)-$\text{Enc}_{2\to 1,\phi}$(16)-LN-Linear(2)} \\
    \\[-1em]\Xhline{2\arrayrulewidth}
\end{tabular}
\end{subtable}\vfill
\vspace{0.4cm}
\begin{subtable}{.99\linewidth}
\caption{Architectures for graph regression experiment.}
\vspace{-0.1cm}
\centering
\begin{adjustbox}{width=0.9\textwidth}
    \label{table:regression_architecture}
    \begin{tabular}{lc}
        \Xhline{2\arrayrulewidth}\\[-1em]
        \\[-1em] Method & Architecture \\
        \\[-1em]\Xhline{2\arrayrulewidth}\\[-1em]
        \\[-1em] Transformer + Laplacian PE & \cameraready{$\text{Enc}_{1\to 1}(256, 16, 16)_{\times8}$-$\text{Enc}_{1\to 0}(256, 16, 16)$-LN-Linear(1)} \\
        \\[-1em] $\text{MLP}_\pi$ (S) & \cameraready{$[L_{2\to 2}(256)-\text{ReLU}]_{\times8}$-$L_{2\to 0}(256)$-Linear(1)} \\
        \Xhline{2\arrayrulewidth}\\[-1em]
        \\[-1em] Ours (S, $\phi$)$_{-\text{SMALL}}$ & \cameraready{$\text{Enc}_{2\to 2,\phi}(256, 8, 4)_{\times8}$-$\text{Enc}_{2\to 0}(256, 16, 8)$-LN-Linear(1)} \\
        \\[-1em] Ours (S, $\phi$) & \cameraready{$\text{Enc}_{2\to 2,\phi}(512, 16, 4)_{\times8}$-$\text{Enc}_{2\to 0}(512, 16, 16)$-LN-Linear(1)} \\
        \\[-1em]\Xhline{2\arrayrulewidth}
    \end{tabular}
\end{adjustbox}
\end{subtable}\vfill
\vspace{0.4cm}
\begin{subtable}{.99\linewidth}
\caption{Architectures for set-to-graph experiment.}
\vspace{-0.1cm}
\centering
\begin{adjustbox}{width=0.9\textwidth}
    \label{table:set2graph_architecture}
        \begin{tabular}{llc}
        \Xhline{2\arrayrulewidth}\\[-1em]
        \\[-1em] Method & Dataset & Architecture \\
        \Xhline{2\arrayrulewidth}\\[-1em]
        \\[-1em]\multirow{3}{7em}{S2G/S2G+ \cite{serviansky2020set}} & Jets &
        \cameraready{$[L_{1\to 1}(256)-\text{ReLU}]_{\times5}$-$L_{1\to 2}(256)$-MLP(1, 256, 1)} \\
        \\[-1em] & Delaunay (50) &
        \cameraready{$[L_{1\to 1}(500)-\text{ReLU}]_{\times7}$-$L_{1\to 2}(500)$-MLP(2, 1000, 1)}\\
        \\[-1em] & Delaunay (20-80) &
        \cameraready{$[L_{1\to 1}(500)-\text{ReLU}]_{\times7}$-$L_{1\to 2}(500)$-MLP(2, 1000, 1)}\\
        \Xhline{2\arrayrulewidth}\\[-1em]
        \\[-1em]\multirow{3}{7em}{Ours (D)} & Jets &
        \cameraready{$\text{Enc}_{1\to 1}(128, 32, 4)_{\times4}$-$\text{Enc}_{1\to 2}(128, 32, 4)$-MLP(1, 256, 1)} \\
        \\[-1em] & Delaunay (50) &
        \cameraready{$\text{Enc}_{1\to 1}(256, 64, 4)_{\times5}$-$\text{Enc}_{1\to 2}(256, 64, 4)$-MLP(2, 256, 1)}\\
        \\[-1em] & Delaunay (20-80) &
        \cameraready{$\text{Enc}_{1\to 1}(256, 64, 4)_{\times5}$-$\text{Enc}_{1\to 2}(256, 64, 4)$-MLP(2, 256, 1)}\\
        \Xhline{2\arrayrulewidth}\\[-1em]
        \\[-1em]\multirow{3}{7em}{Ours (D, $\phi$)} & Jets &
        \cameraready{$\text{Enc}_{1\to 1, \phi}(128, 32, 4)_{\times4}$-$\text{Enc}_{1\to 2, \phi}(128, 32, 4)$-MLP(1, 256, 1)} \\
        \\[-1em] & Delaunay (50) &
        \cameraready{$\text{Enc}_{1\to 1, \phi}(256, 64, 4)_{\times5}$-$\text{Enc}_{1\to 2, \phi}(256, 64, 4)$-MLP(2, 256, 1)}\\
        \\[-1em] & Delaunay (20-80) &
        \cameraready{$\text{Enc}_{1\to 1, \phi}(256, 64, 4)_{\times5}$-$\text{Enc}_{1\to 2, \phi}(256, 64, 4)$-MLP(2, 256, 1)}\\
        \\[-1em]\Xhline{2\arrayrulewidth}
    \end{tabular}
\end{adjustbox}
\end{subtable}\vfill
\vspace{0.4cm}
\begin{subtable}{.99\linewidth}
\caption{Architectures for $k$-uniform hyperedge prediction experiment.}
\vspace{-0.1cm}
\centering
\label{table:hyperedge_architecture}
\begin{adjustbox}{width=0.9\textwidth}
    \begin{tabular}{llc}
        \Xhline{2\arrayrulewidth}\\[-1em]
        \\[-1em] Method & Dataset & Architecture \\
        \Xhline{2\arrayrulewidth}\\[-1em]
        \\[-1em]\multirow{3}{5em}{S2G+ (S)} & GPS &
        \cameraready{$[L_{1\to 1}(64)-\text{ReLU}]_{\times1}$-$L_{1\to 3}(64)$-MLP(4, 64, 1)} \\
        \\[-1em] & MovieLens &
        \cameraready{$[L_{1\to 1}(64)-\text{ReLU}]_{\times3}$-$L_{1\to 3}(64)$-MLP(2, 64, 1)}\\
        \\[-1em] & Drug &
        \cameraready{$[L_{1\to 1}(64)-\text{ReLU}]_{\times3}$-$L_{1\to 3}(64)$-MLP(2, 64, 1)}\\
        \Xhline{2\arrayrulewidth}\\[-1em]
        \\[-1em]\multirow{3}{5em}{Ours (S, $\phi$)} & GPS &
        \cameraready{$\text{Enc}_{1\to 1, \phi}(64, 16, 8)_{\times1}$-$\text{Enc}_{1\to 3, \phi}(64, 16, 8)$-MLP(4, 64, 1)} \\
        \\[-1em] & MovieLens &
        \cameraready{$\text{Enc}_{1\to 1, \phi}(64, 16, 8)_{\times3}$-$\text{Enc}_{1\to 3, \phi}(64, 16, 8)$-MLP(2, 64, 1)}\\
        \\[-1em] & Drug &
        \cameraready{$\text{Enc}_{1\to 1, \phi}(64, 16, 8)_{\times3}$-$\text{Enc}_{1\to 3, \phi}(64, 16, 8)$-MLP(2, 64, 1)}\\
        \\[-1em]\Xhline{2\arrayrulewidth}
    \end{tabular}
\end{adjustbox}
\end{subtable}
\end{table}

In this section, we provide detailed information of the datasets and models used in our experiments in Section~\ref{sec:experiments}.
We provide the dataset statistics in Table~\ref{table:dataset_statistics}, and model architectures in Table~\ref{table:architectures}.

\subsubsection{Implementation details of higher-order Transformers}\label{sec:implementation_details}
\cameraready{
In formulation of higher-order Transformers in the main text, for simplicity we omitted layer normalization (LN) \cite{ba2016layer} and used $\text{ReLU}$ non-linearity for $\text{MLP}_{l\to l}$.
In actual implementation, we adopt Pre-Layer Normalization (PreLN) \cite{xiong2020on}, and place layer normalization before $\text{Attn}_{k\to l}$, before $\text{MLP}_{l\to l}$, and before the output linear projection after the last $\text{Enc}_{k\to l}$.
We also use $\text{GeLU}$ non-linearity \cite{hendrycks2018gaussian} in $\text{MLP}_{l\to l}$ instead of ReLU.
This setup worked robustly in all experiments.
As additional details, we set the internal dimension of $\text{MLP}_{l\to l}$ same as the input and output dimension ($d_F=d$), and applied dropout \cite{srivastava2014dropout} within $\text{Attn}_{k\to l}$ and $\text{MLP}_{l\to l}$ to prevent overfitting.
}

\subsubsection{Efficient implementation of \texorpdfstring{$1\to k$}{1-->k} layers}\label{sec:implementation_details_hyperedge}
\cameraready{
For $k$-uniform hyperedge prediction in Sec.~\ref{sec:experiments}, implementing the higher-order layers $\text{Enc}_{1\to k}$ and $L_{1\to k}$ can be challenging due to the large number of equivalence classes, $\text{b}(1+k)$.
However, we found that it can be reduced to $1+k$ without any approximation.
Specifically, we show the following:
\begin{property}\label{property:1tok}
For $L_{1\to k}$ or $\text{Enc}_{1\to k}$, if we only consider $k$-uniform output hyperedges (output hyperedges without loops; $\mathbf{j}$-th output where $\mathbf{j}_1, ..., \mathbf{j}_k$ are unique), the layers can be implemented using only $1+k$ equivalence classes instead of $b(1+k)$.
\end{property}
\cutsectionup
\begin{proof}
As we only care about output hyperedges with unique index elements, only equivalence classes that correspond to partitions of $[k+1]$ with entries $[k]$ contained in disjoint subsets contribute to output.
There are exactly $1+k$ such partitions depending on which subset the last entry $(k+1)$ belongs to, so it is sufficient that we have $1+k$ equivalence classes.
\end{proof}
\cutsubsectionup
From Property~\ref{property:1tok}, we implement the layers $\text{Enc}_{1\to k}$ and $L_{1\to k}$ by only considering the $1+k$ equivalence classes that contribute to $k$-uniform output hyperedges.
}

\subsubsection{\cameraready{Runtime and memory analysis}}\label{sec:runtime_memory_details}
\cameraready{
For runtime and memory analysis, we used Barabási-Albert random graphs that are made by iteratively adding nodes, where each added node links to 5 random previous nodes.
The experiment was done using a single RTX 6000 GPU with 22GB.
We repeated the experiment 10 times with different random seeds for graph generation and reported the average; variance was generally low.
The architectures of the experimented second-order models are provided in Table~\ref{table:resource_architecture}.
}

\subsubsection{Synthetic chains}\label{sec:chains_details}
For synthetic chains experiment, we used a small dataset composed of 40 training chains each with 20 nodes, and 20 test chains each with 200 nodes \cameraready{as in Table~\ref{table:chain_dataset}}.
Each chain is randomly assigned with a binary label, which is encoded as one-hot vector at a terminal node.
The goal is to classify all nodes in the chain according to the label.
As evaluation metrics, we used macro-/micro-F1 that give combined node-wise F1 scores across all test chains.
All models, including baselines, have fixed hyperparameters with 2 layers and 16 hidden dimensions.
Detailed architectures are provided in Table~\ref{table:chain_architecture}.
For update function of GIN-0, we used an MLP with of 2 hidden layers followed by batchnorm (Linear(16)-ReLU-Linear(16)-ReLU-BN).
For GAT, we used 8 attention heads followed by channelwise sum.
\cameraready{For second-order Transformers, we used a simplified architecture with a single attention head.}
We trained all models with binary cross-entropy loss and Adam optimizer \cite{kingma2015adam} with learning rate 1e-3 and batch size 16 for 100 epochs.

\subsubsection{\cameraready{Large-scale graph regression}}\label{sec:regression_details}
\cameraready{
For large-scale graph regression, we used the PCQM4M-LSC quantum chemistry regression dataset from OGB-LSC benchmark \cite{hu2021ogb}, one of the largest datasets up to date that contains $3.8\text{M}$ molecular graphs.
We provide the summary statistics of the dataset in Table~\ref{table:regression_dataset}.
As the test set is unavailable, we report and compare the Mean Absolute Error (MAE) measured on the validation set.
}

\cameraready{
Table~\ref{table:regression_architecture} gives the architectures of the models used in our experiment.
For second-order models ($\text{MLP}_\pi$ and Ours~(S,~$\phi$)), we used both node and edge types as input information.
For vanilla (first-order) Transformer that operates on node features only, we used Laplacian graph embeddings \cite{belkin2003laplacian, dwivedi2020a} in addition to node types so that the model can consider edge structure information.
The embeddings are computed by factorizing the graph Laplacian matrix \cite{dwivedi2020a}:
\begin{equation}
    \Delta = I - D^{-1/2}AD^{-1/2}=U^\top\Lambda U,
\end{equation}
where $A$ is the adjacency matrix, $D$ is the degree matrix, and $\Lambda, U$ are the eigenvalues and eigenvectors respectively.
Following prior work \cite{dwivedi2020a}, we used the $k$ smallest eigenvectors of a node.
}

\cameraready{
We trained all models with L1 loss using AdamW optimizer \cite{loshchilov2019decoupled} with batch size 1024 on 8 RTX 3090 GPUs.
For all models, we used dropout rate of 0.1 to prevent overfitting.
For the full schedule, we trained our model for 1M steps, and applied linear learning rate warm-up \cite{vaswani2017attention} for 60k steps up to 2e-4 followed by linear decay to 0.
For the short schedule (* in Table~\ref{table:graph_regression}), we trained the models for 100k steps, and applied learning rate warm-up for 5k steps up to 1e-4 followed by decay to 0.
}

\subsubsection{Set-to-graph prediction}\label{sec:set2graph_details}
For set-to-graph prediction experiment, we borrow the datasets, code, and baseline scores from Serviansky~et.~al.~(2020)~\cite{serviansky2020set}.
\cameraready{We provide the summary statistics of the datasets in Table~\ref{table:set2graph_dataset}.}

\cameraready{
As in main text, Jets is a dataset where the task is to infer partition of a set of observed particles.
By viewing each partition as a fully-connected graph, the task becomes graph prediction problem.
Each data instance contains 2-14 nodes, each having 10-dimensional features.
The entire dataset contains 0.9M instances, divided into 60/20/20\% train/val/test sets.
Evaluation is done with 3 metrics: F1 score, Rand Index (RI), and Adjusted Rand Index (ARI) which is computed as $\text{ARI} = (\text{RI}-\mathbb{E}[\text{RI}])/(1-\mathbb{E}[\text{RI}])$.
To ensure that the model's prediction gives a correct partitioning, a postprocessing is applied to convert every connected components to cliques.
The test set is further separated into 3 types: bossom(B)/charm(C)/light(L), depending on underlying data generation process.
This makes typical \# of partitions in each set different.
Among the baselines, GNN is a message-passing GNN \cite{gilmer2017neural} that operate on k-NN induced graph for $k=5$, where edge prediction is done with pairwise dot-product.
AVR is an algorithmic baseline typically used in particle physics.
}

\cameraready{
As in main text, Delaunay datasets involve 2D point sets where the task is performing Delaunay triangulation.
Evaluation metrics are typical Accuracy/Precision/Recall/F1 scores based on edge-wise binary classification on held-out test set.
The baselines are similar to Jets; GNN0/5/10 are message-passing GNNs \cite{gilmer2017neural} that operate on k-NN induced graph for $k\in\{0, 5, 10\}$.
}

\cameraready{
Table~\ref{table:set2graph_architecture} provides the architecture of the models used in our experiment, along with relevant baselines S2G/S2G+ from Serviansky~et.~al.~(2020)~\cite{serviansky2020set}.
S2G uses a subset of equivalence classes~($\mu$) within $L_{1\to 2}$, and S2G+ uses full basis\footnote{Note that the implementation of linear layers in Serviansky~et.~al.~(2020)~\cite{serviansky2020set} is slightly different from ours.}.
Our models, both (D) and (D,~$\phi$), are made by substituting $\text{Enc}_{1\to 1}$ and $\text{Enc}_{1\to 2}$ into S2G+.
All models were trained with Adam optimizer to minimize the combination of soft F1 score and binary cross-entropy of edge prediction.
For all models, we used dropout rate of 0.1 to prevent overfitting.
For Jets, with 400 max epochs, the training is early-stopped based on validation F1 score with 20-epoch tolerance.
We used learning rate 1e-4 and batch size 512 for our models, while S2G/S2G+ used learning rate 1e-3 and batch size 2048 \cite{serviansky2020set}.
For Delaunay, with 100 max epochs, we used learning rate 1e-4 for our models, and used batch size 32/16 for Delaunay~(50)/(20-80);
S2G/S2G+ used learning rate 1e-3 and batch size 64 \cite{serviansky2020set}.
For Ours~(D,~$\phi$) in Delaunay (20-80), we applied 1-epoch warmup to prevent early training instability.
}

\subsubsection{\cameraready{\texorpdfstring{$k$}{k}-uniform hyperedge prediction}}\label{sec:hyperedge_details}
\cameraready{
For $k$-uniform hyperedge prediction experiment, we borrow the datasets, code, and baseline scores from Zhang~et.~al.~(2020)~\cite{zhang2020hyper}.
As in the main text, we used three datasets for transductive 3-edge prediction.
The first dataset GPS contains (user-location-activity) hyperedges.
The second dataset MovieLens contains (user-movie-tag) hyperedges.
The third dataset Drug contains (user-drug-reaction) hyperedges.
We provide the summary statistics in Table~\ref{table:chain_dataset}.
}

\cameraready{
The experiments were done in a transductive setup, where the hyperedge set is randomly split into the training and test set with 4:1 ratio.
We randomly sampled negative edges to be 5 times the amount of the positive edges, so that hyperedge prediction becomes binary classification problem.
Thus, the evaluation is done with AUC and AUPR scores.
Among the baselines, for Hyper-SAGNN, we reproduced the scores using the open-sourced code \cite{zhang2020hyper} using the provided hyperparameters.
For additional baselines including node2vec, we take the scores reported in Zhang~et.~al.~(2020)~\cite{zhang2020hyper}.
}

\cameraready{
Table~\ref{table:hyperedge_architecture} gives the architecture of the models used in our experiment.
As Hyper-SAGNN uses autoencoder-based node features, for proper comparison we also adopted and trained them jointly with the full model \cite{zhang2020hyper}.
All models (including reproduced Hyper-SAGNN) were trained with Adam optimizer to minimize the combination of binary cross-entropy loss and autoencoder reconstruction loss for 300 epochs with learning rate 1e-3 and batch size 96.
For S2G+~(S) and Ours~(S,~$\phi$), we applied dropout rate of 0.1 to the hidden layers of MLP after $L_{1\to3}$ or $\text{Enc}_{1\to3}$ to prevent overfitting.
}

\subsection{Potential negative social impacts}
Our framework can be potentially applied to a variety of tasks involving relational data, e.g., molecular structures, social networks, 3D mesh, etc.
Advancements in those directions might incur negative side-effects such as low-cost biochemical weapon, deepening of filter bubbles from enhanced personalized social network services, surveillance with mesh-based face recognition, etc.
Such potential negative impacts should be addressed as we conduct domain-specific follow-up works.


\end{document}